\pdfoutput=1
\documentclass[letterpaper]{article}
\usepackage{uai2019}
\usepackage{custom}

\usepackage[margin=1in]{geometry}
\usepackage{times}

\title{Non-Canonical Hamiltonian Monte Carlo}

\author{} 

\author{ {\bf James A. Brofos} \\
Yale University \\
\And
{\bf Roy R. Lederman}  \\
Yale University \\
}

\begin{document}

\maketitle

\begin{abstract}
Hamiltonian Monte Carlo is typically based on the assumption of an underlying canonical symplectic structure. Numerical integrators designed for the canonical structure are incompatible with motion generated by non-canonical dynamics. These non-canonical dynamics, motivated by examples in physics and symplectic geometry, correspond to techniques such as  preconditioning which are routinely used to improve algorithmic performance. Indeed, recently, a special case of non-canonical structure, magnetic Hamiltonian Monte Carlo, was demonstrated to provide advantageous sampling properties. We present a framework for Hamiltonian Monte Carlo using non-canonical symplectic structures. Our experimental results demonstrate sampling advantages associated to Hamiltonian Monte Carlo with non-canonical structure.  To summarize our contributions: (i) we develop non-canonical HMC from foundations in symplectic geomtry; (ii) we construct an HMC procedure using implicit integration that satisfies the detailed balance; (iii) we propose to accelerate the sampling using an {\em approximate} explicit methodology; (iv) we study two novel, randomly-generated non-canonical structures: magnetic momentum and the coupled magnet structure, with implicit and explicit integration.
\end{abstract}

\blfootnote{Approved for Public Release; Distribution Unlimited. Public Release Case Number 20-1337. James Brofos's affiliation with The MITRE Corporation is provided for identification purposes only, and is not intended to convey or imply MITRE's concurrence with, or support for, the positions, opinions, or viewpoints expressed by the author. \copyright 2020 The MITRE Corporation and the authors. All rights reserved.}

\section{Introduction}

Bayesian inference provides a mechanism to capture uncertainties in complex statistical models but is complicated by intractable normalizing constants and multi-modal densities. A state-of-the-art method for generating samples from differentiable multi-dimensional distributions is Hamiltonian Monte Carlo (HMC) \citep{1206.1901,10.5555/922680}. HMC is able to leverage gradient information and Hamilton's equations of motion in order to propose distant candidate samples; contrast this with random walk Monte Carlo or Metropolis-adjusted Langevin dynamics whose proposal distribution is centered at the current state. The effect of this ability to propose far-away samples is a reduction in sample auto-correlation and higher effective sample sizes than competing methods.

\begin{figure}[t!]
    \centering
    \begin{subfigure}[b]{0.15\textwidth}
        \centering
        \includegraphics[width=\textwidth]{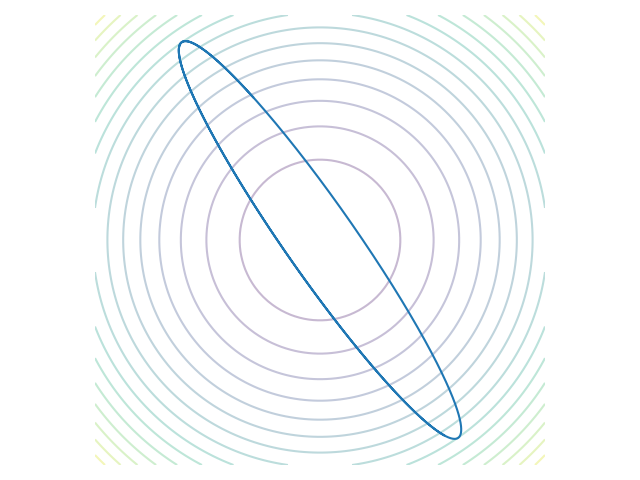}
        \caption{}
    \end{subfigure}
    ~
    \begin{subfigure}[b]{0.15\textwidth}
        \centering
        \includegraphics[width=\textwidth]{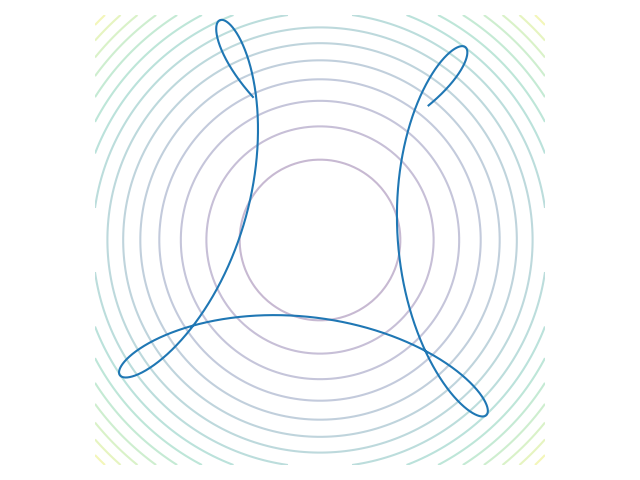}
        \caption{}
    \end{subfigure}
    ~
    \begin{subfigure}[b]{0.15\textwidth}
        \centering
        \includegraphics[width=\textwidth]{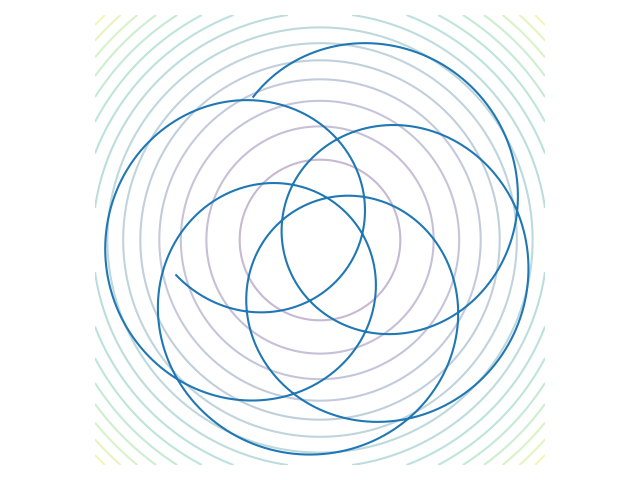}
        \caption{}
    \end{subfigure}
    \caption{Visualization of the position variable for (a) canonical dynamics, (b) magnetic dynamics, and (c) fully non-canonical dynamics. Particles evolve according to a Gaussian (quadratic) Hamiltonian, but are subject to distinct symplectic structures which influence particle trajectories.}
    \label{fig:non-canonical-visualization}
\end{figure}

In this work, we examine the geometric foundations of Hamiltonian dynamics and consider non-canonical dynamics. Recent work in \citep{pmlr-v70-tripuraneni17a} established a version of non-canonical dynamics which they called magnetic Hamiltonian Monte Carlo, and introduced an explicit integrator for non-canonical dynamics with a magnetic physical intuition. Yet non-canonical dynamics encompass a broader scope than motion under the influence of a magnetic field alone: {\it our current paper is motivated by the observation that non-canonical dynamics provide a rich (and mostly unexplored) class of inference procedures}. We expand upon prior work  using a an {\it approximate} explicit integrator for non-separable Hamiltonians on non-canonical symplectic vector spaces. Our principle theoretic tool is Darboux's Theorem, which permits a change-of-basis in which the non-canonical structure assumes the canonical form. 
{\it Although this explicit integrator yields a proposal operator that does not satisfy the exact detailed balance, we compare this approximation to the implicit integration methods (that does satisfy the detailed balance) and find that the explicit method is faithful to the posterior,} and computationally faster. 
We evaluate non-canonical HMC with implicit integration and our explicit approximation on a general class of non-canonical structure. Our experimental results demonstrate that non-canonical dynamics yield more efficient samples relative to competing HMC methods and that explicit integration can further accelerate the method.

The structure of this paper is as follows. In \cref{sec:preliminaries} we provide mathematical background on Hamiltonian mechanics in the context of HMC, including the construction of Hamiltonian mechanics from symplectic geometry. Additional background on HMC can be found in \cref{app:hamiltonian-monte-carlo}.
In \cref{sec:analytical-apparatus} we review known results on non-canonical Hamiltonian dynamics and Darboux's Theorem. \Cref{sec:implicit-methods-for-non-canonical-hmc} discusses implicit integration methods for non-canonical HMC including the fact that proposal operators built from implicit midpoint integration satisfies detailed balance. Additional details about this numerical integration can be found in \cref{app:numerical-integration}.
\Cref{sec:explicit-acceleration} discusses our procedure for accelerating the integration of non-canonical Hamiltonian dynamics using Darboux's Theorem and related tools from symplectic geometry. 
Additional technical details about the numerical procedures associated with Darboux's Theorem and the explicit numerical integration can be found in \cref{app:symplectic-gram-schmidt,app:explicit-integration-scheme,app:remarks-on-proposal-operator,app:binding-strength-parameter,app:implicit-reversibility}.
Our experimental results are shown in \cref{sec:experiments} on two benchmark inference tasks. We include detailed appendices that touch on many aspects of Hamiltonian Monte Carlo, non-canonical dynamics, and integration strategies that serve to supplement the core paper. \Cref{app:proof-non-canonical-dynamics-darboux,app:proof-time-reversal-poisson} present proofs of results used in the paper.

\section{Preliminaries}\label{sec:preliminaries}

In this section we describe the underlying geometric concepts for HMC. We start with notation and proceed to formalize Hamilton's equations of motion, describing their fundamental properties, their construction in symplectic vector spaces, and the accompanying Poisson structure. We subsequently discuss symplectic integration methods.

\subsection{Notation}

Let $\text{Id}$ represent the $n\times n$ identity matrix and $\mathbf{0}$ the $n\times n$ zero matrix. Let $\text{Skew}(n)$ denote the set of $n\times n$ skew-symmetric matrices and $\text{Symm}(n)$ the set of $n\times n$ symmetric matrices. We define the canonical symplectic matrix
\begin{align} \label{eq:jcanon}
    \mathbb{J}_\text{can} = \begin{bmatrix} \mathbf{0} & \text{Id} \\ -\text{Id} &\mathbf{0}\end{bmatrix} \in \text{Skew}(2n).
\end{align}
Given a vector space $V$ of dimension $n$ and basis $\set{e^1,\ldots, e^n}$ consider $z=\sum_{i=1}^n a_ie^i$. We will sometimes use the basis isomorphism to identify $z\equiv (a_1,\ldots, a_n)$. We will write $(\mathbf{x},\mathbf{y})$ to represent the concatenation of vectors $\mathbf{x}$ and $\mathbf{y}$, regarded as a row vector so that $(\mathbf{x},\mathbf{y})^\top$ is a column vector.

\subsection{Mathematical Background}

In our discussion, the vectors space will usually take the form $Z = V\times V^*$, 
where $V = \mathbb{R}^n$, $V^* = \mathbb{R}^n$ and $Z=\mathbb{R}^{2n}$.
\begin{definition}[Symplectic Structure]
A symplectic structure $\Omega : Z\times Z\to \R$ is the skew-symmetric bilinear operator.
\end{definition}
The matrix $\mathbb{J}_\Omega\in\text{Skew}(2n)$ associated with the symplectic structure $\Omega$ is a skew-symmetric matrix
such that $\Omega(u, v) = (\mathbb{J}_\Omega^\top u)(v)$. 
When $\mathbb{J}_\Omega^\top$ is an isomorphism of $Z$ and $Z^*$, we say the symplectic structure $\Omega$ is {\em non-degenerate}. All the symplectic forms in our discussion will be assumed to be non-degenerate.


\begin{definition}[Symplectic Vector Space]\label{def:symspace}
Let $Z$ be a vector space  equipped with a non-degenerate skew-symmetric bilinear form $\Omega$. Then $(Z, \Omega)$ is called a symplectic vector space.
\end{definition}

\subsection{Hamilton's equations: Canonical Separable Case}

We begin by reviewing the mathematical fundamentals of Hamiltonian mechanics in the context of HMC. Our presentation follows the discussion in \citep{marsden2002introduction}. Let $V$ be a vector space of dimension $\abs{V}=n$, for a positive integer $n$. Let $\set{e_1, \ldots, e_n}$ be a basis for $V$ and $\set{e^1,\ldots, e^n}$ a basis for the dual space $V^*$. Hence a vector $(u, v)\in V\times V^*$ may naturally be identified with a coordinate representation $(q^1,\ldots,q^n, p_1,\ldots,p_n)$ such that $u = \sum_{i=1}^n q^i e_i$ and $v = \sum_{i=1}^n p_i e^i$. A point $(\mathbf{q}, \mathbf{p}) = (q^1,\ldots, q^n, p_1,\ldots, p_n)$ is a point in the phase-space of the system. The variables $\mathbf{q}$ and $\mathbf{p}$ are called position and momentum, respectively.

Hamilton's canonical equations describe the motion of a particle in a mechanical system whose total energy is given by the Hamiltonian $H(\mathbf{q}, \mathbf{p})$; we will forthwith assume that $H$ is differentiable in both of its arguments. 

A Hamiltonian is called {\em separable} if it may be decomposed into the sum of terms, one dependent on the position variables alone and the other dependent on the momentum variables alone; mathematically, $H(\mathbf{q}, \mathbf{p}) = U(\mathbf{q}) + K(\mathbf{p})$. The function $U(\mathbf{q})$ is called the potential energy while $K(\mathbf{p})$ is the kinetic energy, which is often of quadratic form $K(\mathbf{p}) = \mathbf{p}^\top\mathbf{p} / 2$. 

The evolution in phase-space coordinates is given by
\begin{align} \label{eq:canonical-motion}
    \frac{\d{}}{\d{t}} \mathbf{q} = \nabla_\mathbf{p} H(\mathbf{q}, \mathbf{p}) ~~~~~~~ \frac{\d{}}{\d{t}} \mathbf{p} = -\nabla_\mathbf{q} H(\mathbf{q}, \mathbf{p}).
\end{align}
The canonical Hamiltonian dynamics exhibit three appealing properties:
\begin{enumerate}[i]
    \item The Hamiltonian is conserved in time-evolution $\frac{\d{}}{\d{t}} H(\mathbf{q}, \mathbf{p}) = 0$.
    \item Volume in phase-space is preserved in time-evolution such that if $R_0$ is a subset of phase-space and $R_t$ is its image under Hamilton's equation of motion at time $t>0$ then $\text{Vol}(R_0) = \text{Vol}(R_t)$.
    \item Hamiltonian dynamics are reversible in time and, if the Hamiltonian satisfies $H(\mathbf{q}, \mathbf{p}) = H(\mathbf{q}, -\mathbf{p})$, the reversal may be achieved by negating the sign of the momentum $\mathbf{p} \to -\mathbf{p}$.
\end{enumerate}
Remarkably, the first and second of these properties will continue to hold even for non-canonical Hamiltonian dynamics. The third property is more subtle: the dynamics are reversible by the flow property of differential equations; however, a more complicated reversal procedure than flipping the sign of the momentum will be necessary in order to realize the time reversal effect. Further details may be found in \cite{pmlr-v70-tripuraneni17a} and \cref{app:proof-time-reversal-poisson,app:implicit-reversibility}

Hamilton's equations may be elegantly considered by endowing $Z = V\times V^*$ with the {\em canonical symplectic structure} $\Omega_\text{can}$, defined by the relation  $\Omega_\text{can}(u, v) \defeq u^\top \mathbb{J}_\text{can} v$ for $u,v \in V\times V^*$, where $\mathbb{J}_\text{can}$ is defined in \cref{eq:jcanon}.

Let $z=(p,q)$ be coordinates in phase-space. 
Denote by $\mathbf{D}H$ the concatenation of the derivatives with respect to $p$ and $q$:
\begin{equation}\label{eq:d}
\mathbf{D}H = \paren{\nabla_\mathbf{q} H, \nabla_\mathbf{p} H}.
\end{equation}
In the mathematical formalism of the symplectic vector space, Hamilton's canonical equations of motion in Equation \cref{eq:canonical-motion} are expressible by the Hamiltonian vector field $X_H : Z\to Z$, defined by the formula
\begin{equation}\label{eq:hamiltonian-vector-field}
    X_H(z) = (-\mathbb{J}_\text{can})^{-1} ~\mathbf{D}H(\mathbf{q}, \mathbf{p}) = \mathbb{J}_\text{can} ~\mathbf{D}H(\mathbf{q}, \mathbf{p}),
\end{equation}
where, $\mathbf{D}$ is defined in \cref{eq:d}, and in the canonical case $\mathbb{J}_\Omega = \mathbb{J}_\text{can}$ defined in \cref{eq:jcanon}. Note that \cref{eq:hamiltonian-vector-field} represents a vector field since it is a (smooth) assignment of a vector $X_H(z)\in Z$ for each $z\in Z$. 

If $z\in Z$ evolves according to these dynamics, then we obtain
\begin{equation}
\frac{\d{}}{\d{t}} z = X_H(z),
\end{equation}
which is equivalent to \cref{eq:canonical-motion}.

\subsection{Hamilton's Equations: Non-Canonical Case}

More generally, the Hamiltonian vector field in \cref{eq:hamiltonian-vector-field}
on a symplectic vector space (see \cref{def:symspace}) is defined by the relation $\Omega(X_H(z), \delta) = \mathbf{D}H(z)\cdot \delta$ for $z,\delta\in Z$, with $\mathbf{D}$ defined in \cref{eq:d}. Since $\Omega$ is non-degenerate and its matrix is skew-symmetric, this relationship assumes the following matrix expression: $(X_H(z))^\top \mathbb{J}_\Omega \delta = \mathbf{D}H(z)^\top \delta$ or 
\begin{equation}\label{eq:general-hamiltonian-vector-field}
    X_H(z) = (-\mathbb{J}_\Omega)^{-1} ~\mathbf{D}H(z).
\end{equation}

The {\em Poisson matrix} $\mathbb{B}_\Omega$ associated with the symplectic matrix $\mathbb{J}_\Omega$ is the inverse of the transpose of the symplectic matrix $\mathbb{B}_\Omega \defeq (-\mathbb{J}_\Omega)^{-1}$ (where $\mathbb{J}_\Omega^\top = -\mathbb{J}_\Omega$ since $\mathbb{J}_\Omega$ is skew-symmetric). From the relation (\ref{eq:general-hamiltonian-vector-field}) we immediately have the equivalent statement 
\begin{equation}\label{eq:general-hamiltonian-vector-field-B}
X_H(z) = \mathbb{B}_\Omega ~\mathbf{D}H(z).     
\end{equation}
Note that because $\mathbb{J}_\Omega$ is skew-symmetric so is $\mathbb{B}_\Omega$. For our purposes, it will be convenient to write $\mathbb{B}_\Omega$ as a block of four matrices like so
\begin{align}\label{eq:poisson-matrix}
    \mathbb{B}_\Omega = \begin{bmatrix} \mathbf{E} & \mathbf{A} \\ -\mathbf{A}^\top & \mathbf{G} \end{bmatrix}
\end{align}
where $\mathbf{E}, \mathbf{G}\in \text{Skew}(n)$, where we have suppressed the dependence of $\mathbf{E},\mathbf{A},\mathbf{G}$ on $\Omega$ for notational brevity. 

In a system with a separable Hamiltonian comprised of a potential and kinetic energy, the non-canonical Poisson matrix $\mathbb{B}_\Omega$ enables gradients of the potential energy to flow into the time-derivative of the state variable, while gradients of the kinetic energy may similarly flow into the time-derivative of the momenta. Just as $\Omega$ denotes the symplectic structure whose matrix is $\mathbb{J}_\Omega$ we will use $\Lambda_\Omega : Z\times Z\to \R$ to represent the {\em Poisson structure} whose matrix is $\mathbb{B}_\Omega$: $\Lambda_\Omega(u, v) = u^\top \mathbb{B}_\Omega v$ for $u,v \in Z$.

\section{Analytical Apparatus}\label{sec:analytical-apparatus}

This section pertains to theoretical and practical considerations for non-canonical Hamiltonian dynamics. 

\subsection{Reversibility of Non-Canonical Hamiltonian Dynamics}

For now, we will recall some facts about non-canonical Hamiltonian dynamics that apply for an arbitrary Poisson structure.  In canonical Hamiltonian dynamics, time reversal can be achieved by reversing the momentum as in property (iii); in the non-canonical case, a more involved procedure is required. 
The following theorem from \citep{pmlr-v70-tripuraneni17a} gives the analogue of time-reversibility for non-canonical dynamics.
\begin{theorem}\label{thm:time-reversal-poisson}
Consider a Poisson structure $\Lambda_\Omega$ with matrix $\mathbb{B}_\Omega$ defined in \cref{eq:poisson-matrix} and let $X_H(z) = \mathbb{B}_\Omega ~\mathbf{D}H(\mathbf{q}, \mathbf{p})$ be the corresponding Hamiltonian vector field. Unlike canonical Hamiltonian dynamics, reversing the sign of the momentum variable is not sufficient to reverse the direction of time. However, the dynamics with augmented Poisson structure $\tilde{\Lambda}_\Omega$ whose matrix is
\begin{align}\label{eq:brev}
    \tilde{\mathbb{B}}_\Omega = \begin{bmatrix} -\mathbf{E} & \mathbf{A} \\ -\mathbf{A}^\top & -\mathbf{G} \end{bmatrix}
\end{align}
will have the time-reversal effect when the initial condition is $(\mathbf{q}, -\mathbf{p})$. We call $\tilde{\Lambda}_\Omega$ the time-reversal Poisson structure.
\end{theorem}
The proof of the Theorem can be found in \citep{pmlr-v70-tripuraneni17a}. 
We give a generalization of this theorem to the case of state-dependent Poisson structure in \cref{app:proof-time-reversal-poisson}, although in our experiments and theoretical treatment only the case of a constant Poisson matrix is considered.

\subsection{Magnetic Hamiltonian Monte Carlo}

Most relevant to our research is magnetic HMC \citep{pmlr-v70-tripuraneni17a}. Magnetic HMC considers a non-canonical Poisson $\Lambda_\text{mag}$ structure whose matrix in coordinates assumes the form
\begin{align}
    \mathbb{B}_\text{mag} = \begin{bmatrix} \mathbf{0} & \mathbf{A} \\ -\mathbf{A}^\top & \mathbf{G} \end{bmatrix} \in \text{Skew}(2n).
\end{align}
Recalling the general Poisson structure in \cref{eq:poisson-matrix}, this formulation corresponds to $\mathbf{E}=\mathbf{0}$. This form is motivated by a physical intuition in the special case of $n=3$, wherein this Poisson structure describes the motion of a particle under the influence of a magnetic field. We refer to this non-canonical structure as a magnetic position structure. For this special case of non-canonical Poisson structure, there exists an explicit leapfrog integrator that can be used as a transition operator. One of our contributions in this work is to develop an explicit integrator for the case of $\mathbf{E}\neq\mathbf{0}$.

\subsection{Reduction to Canonical Form via Darboux's Theorem}

Our purpose in \cref{sec:explicit-acceleration} is to illustrate how to design an explicit, symplectic integration strategy for Hamiltonian dynamics in the setting of a non-canonical Poisson structure. Our first step will be to use a basis transform such that, in the new basis, the matrix of the symplectic structure is canonical. The existence of such a basis is a consequence of Darboux's Theorem.

\begin{theorem}[Darboux's Theorem for a Symplectic Vector Space]\label{thm:darboux-theorem}
If $V$ is a symplectic vector space with non-degenerate symplectic form $\Omega$, then there exists a basis $\mathbf{B}$ and change-of-basis matrix $\mathbf{F} \defeq \mathbf{B}^{-1} : V\times V^*\to V\times V^*$ such that, in coordinates $\tilde{z} \defeq (\tilde{\mathbf{q}},\tilde{\mathbf{p}}) =  (\tilde{q}^1,\ldots, \tilde{q}^n, \tilde{p}_1, \ldots, \tilde{p}_n)$ of the new basis, the symplectic structure $\Omega$ is canonical.
\end{theorem}

\section{Implicit Methods for Non-Canonical Hamiltonian Monte Carlo}\label{sec:implicit-methods-for-non-canonical-hmc}

\begin{algorithm}[ht!]
\caption{The procedure for implementing non-canonical Hamiltonian Monte Carlo using an explicit integrator. It is necessary to perform an intricate reversal of symplectic structure and momentum in order to maintain symmetry of the proposal.}
\label{alg:non-canonical-hmc}
\begin{algorithmic}[1]
\State \textbf{Parameters}: Bonding strength $\omega >0$ and step-size $\epsilon > 0$. Number of integration steps $N$. Decision to use implicit or explicit integration.
\State \textbf{Input}: Starting position variable $\mathbf{q}$; non-canonical symplectic structure $\Omega$ with matrix $\mathbb{J}$ in non-canonical coordinates. Separable Hamiltonian $H(z) \equiv H(\mathbf{q}, \mathbf{p})$.
\State \textbf{Output}: Samples from a Markov chain targeting $p(\mathbf{q}, \mathbf{p}) \propto \exp(-H(\mathbf{q}, \mathbf{p}))$.
\State Compute Poisson matrix $\mathbb{B} = -\mathbb{J}^{-1}$, the time reversal Poisson structure $\tilde{\mathbb{B}}$ and time-reversal symplectic structure with matrix $\tilde{\mathbb{J}} = -\tilde{\mathbb{B}}^{-1}$ from \cref{thm:time-reversal-poisson}.
\State Use symplectic Gram-Schmidt (see \cref{app:symplectic-gram-schmidt}) to find base $\mathbf{B}$ and its time-reversal $\tilde{\mathbf{B}}$ (\cref{eq:brev}) in which the symplectic structure and time-reversal symplectic structure, respectively, have canonical form. Set $\mathbf{F} = \mathbf{B}^{-1}$.
\While{Not done sampling}
    \State Sample $\mathbf{p}\sim\mathcal{N}(\mathbf{0}, \text{Id})\in \R^n$.
    \If{Implicit integration}
        \State Integrate the dynamics corresponding to $H(\mathbf{q},\mathbf{p})$ with the non-canonical Poisson structure $\mathbb{B}$ using the implicit midpoint integrator; see \cref{app:numerical-integration}. {\it Use this method for an exact, but slower, MCMC sampler.}
    \ElsIf{Explicit integration}
        \State  Integrate the dynamics corresponding to (possibly non-separable) Hamiltonian $\tilde{H}(\tilde{\mathbf{q}}, \tilde{\mathbf{p}}) = H(\mathbf{B}(\tilde{\mathbf{q}}, \tilde{\mathbf{p}})^\top)$ for $N$ steps using the explicit integrator (see \cref{sec:explicit-acceleration,app:explicit-integration-scheme}) with step-size $\epsilon$. This computes candidate $(\tilde{\mathbf{q}}',\tilde{\mathbf{p}}')$ in canonical coordinates. Convert candidate point in canonical coordinates back to non-canonical coordinates $(\mathbf{q}', \mathbf{p}')^\top = \mathbf{B}(\tilde{\mathbf{q}}',\tilde{\mathbf{p}}')^\top$. {\it Use this method for an approximate, but faster, MCMC sampler.}
    \EndIf
    \State Compute Metropolis-Hastings correction $M = \min(0, H(\mathbf{q}, \mathbf{p}) - H(\mathbf{q}', \mathbf{p}'))$ and sample $u\sim\text{Uniform}(0, 1)$. 
    \If{$\log u < M$}
        \State Set $(\mathbf{q}, \mathbf{p}) = (\mathbf{q}', \mathbf{p}')$.
    \Else
        \State \textbf{Optional}: Set $\mathbf{F} = \mathbf{B}^{-1}$ if $\mathbf{F} = \tilde{\mathbf{B}}^{-1}$; otherwise set $\mathbf{F} = \tilde{\mathbf{B}}^{-1}$. Set $\epsilon=-\epsilon$.
    \EndIf
    \State \textbf{Yield}: $(\mathbf{q}, \mathbf{p})$.
\EndWhile
\end{algorithmic}
\end{algorithm}

Our procedure for non-canonical Hamiltonian Monte Carlo differs only in two respects from the classic HMC discussed in \cref{app:hamiltonian-monte-carlo}. First, we generate transitions according to equations of motion given by a non-canonical Poisson matrix $\mathbb{B}_\Omega$ corresponding to a non-canonical symplectic structure $\Omega$ and, second, our method of integrating these equations of motion uses the implicit midpoint integrator rather than the leapfrog procedure; this yields a proposal for which a Metropolis-Hastings accept-reject decision is applied, giving a reversible proposal operator. 

Having stated that non-canonical HMC with implicit integration is correct Monte Carlo, we now establish detailed balance for non-canonical HMC with implicit integration.

\begin{theorem}\label{thm:non-canonical-detailed-balance}
Let $\Lambda_\Omega$ be a Poisson structure corresponding to the symplectic structure $\Omega$ with time-reversal Poisson structure $\tilde{\Lambda}_\Omega$. Suppose $H(\mathbf{q},\mathbf{p}) = U(\mathbf{q}) + \frac{1}{2} \mathbf{p}^\top\mathbf{p}$ is a separable Hamiltonian. The implicit midpoint integrator satisfies detailed balance with respect to the Gibbs distribution proportional to $\exp\paren{-H(\mathbf{q},\mathbf{p})}$ when combined with a momentum-flip operator and a transition to the time-reversal Poisson structure as described in \cref{thm:time-reversal-poisson}.
\end{theorem}
A proof is given in \cref{app:implicit-reversibility}. Detailed balance quickly follows from the fact that the implicit midpoint integrator is reversible and symplectic (so that it is volume preserving).

\section{Acceleration via Explicit Integration}\label{sec:explicit-acceleration}

HMC with transitions computed from implicit integration leaves the stationary distribution invariant. However, being implicit, we expect HMC based on these transitions to be relatively slow compared to explicit integrators. {\it This observation motivates our interest in developing explicit integrators that can be utilized in place of implicit integration.} The following approach achieves this, but at the cost of provable detailed balance. 

We propose an {\it approximate} method for non-canonical HMC via an explicit integrator for the same dynamics. Explicit integration is achieved by transforming the non-canonical equations of motion into a basis in which the symplectic structure appears canonical, integrating the canonical equations of motion in this basis, and performing a change-of-basis back to non-canonical coordinates at the end of the trajectory. 

\subsection{Designing an Integrator: Implementing the Reduction to Canonical Form}

Darboux's Theorem (Theorem \ref{thm:darboux-theorem}) guarantees the existence of a basis in which the non-canonical symplectic structure assumes the canonical form, but how will one find the change-of-basis matrix $\mathbf{F}$? An explicit procedure, known as symplectic Gram-Schmidt is given in \cref{app:symplectic-gram-schmidt}, based off the procedure in \citep{gole2001symplectic}. The technique is to identify pairs of vectors defining 2-dimensional subspaces whereupon the symplectic structure assumes the canonical form; by building up collections of orthogonal symplectic planes, one identifies a basis of the whole linear space for which the symplectic structure becomes canonical. Note that, unlike traditional Gram-Schmidt from which resulting vectors must have unit norm, the basis produced by symplectic Gram-Schmidt can have basis vectors whose norm exhibit large dynamic range. In our experiments we prefer the symplectic basis whose Frobenius norm is smallest. In practice, one may run symplectic Gram-Schmidt multiple times with distinct random seeds in order to identify a symplectic basis with small norm.

The symplectic Gram-Schmidt procedure may not always be necessary, however, and usable symplectic bases may be discovered by inspection of the symplectic structure. \Cref{tab:symplectic-bases} shows non-canonical structures and how to construct a basis in which the symplectic structure assumes canonical form. In the new basis, the Hamiltonian $H(\mathbf{q},\mathbf{p})$ is naturally reformed as $\tilde{H}(\tilde{\mathbf{q}},\tilde{\mathbf{p}}) \defeq H(\mathbf{B}(\tilde{\mathbf{q}},\tilde{\mathbf{p}})^\top)$.

One wonders if this procedure of changing bases results in different motion, in some sense. If the Hamiltonian is modified appropriately under the change-of-basis, the answer is no. The following theorem establishes the equivalency of the dynamics with non-canonical Poisson structure $\Lambda_\Omega$ under Hamiltonian $H(\mathbf{q},\mathbf{p})$ and the dynamics in canonical coordinates with Hamiltonian $\tilde{H}(\tilde{\mathbf{q}},\tilde{\mathbf{p}})$.

\begin{table*}[ht!]
\centering
\begin{tabular}{lccl}
                                                                                        & Poisson form $\mathbb{B}_\Omega$                                                           & Change of Basis $\mathbf{B}$                                                                                                                                                & Notes                                                                                                                           \\ \hline
\multicolumn{1}{l|}{Mass Preconditioning}                                               & $\begin{bmatrix} \mathbf{0} & \mathbf{M} \\ -\mathbf{M}^\top &\mathbf{0}\end{bmatrix}$ & $\begin{bmatrix}\mathbf{L} &  \mathbf{0} \\ \mathbf{0} & \mathbf{L}\end{bmatrix}$                                                                                           & $\mathbf{M} \succ \mathbf{0}$ and $\mathbf{L} = \sqrt{\mathbf{M}}$ \\
\multicolumn{1}{l|}{Magnetic Position}                                                           & $\begin{bmatrix} \mathbf{0} & \text{Id} \\ -\text{Id} & \mathbf{G}\end{bmatrix}$       & $\begin{bmatrix} \text{Id} & \mathbf{0} \\ \mathbf{A} & \text{Id} \end{bmatrix}$                                                                                            & $\mathbf{A} = \mathbf{G} / 2$                                                                                                   \\
\multicolumn{1}{l|}{\begin{tabular}[c]{@{}l@{}}Magnetic Position\\ Preconditioning\end{tabular}} & $\begin{bmatrix}\mathbf{0} & \mathbf{M} \\ -\mathbf{M}^\top & \mathbf{G}\end{bmatrix}$ & $ \begin{bmatrix} \mathbf{L} & \mathbf{0} \\ \mathbf{0} & \mathbf{L}\end{bmatrix}\begin{bmatrix} \text{Id} & \mathbf{0} \\ \mathbf{Q} & \text{Id} \end{bmatrix}$ & $\mathbf{Q} = \frac{(\mathbf{L}^{-1})^\top \mathbf{G} \mathbf{L}^{-1}}{2}$. \\ \hline\hline                
\end{tabular}
\caption{Symplectic Gram-Schmidt gives a procedure for finding a symplectic basis given any non-degenerate symplectic form. This procedure is not always necessary, however. One can sometimes find a symplectic basis by inspection of the symplectic form and its inverse transformation $\mathbf{F}$ may be forthrightly derived also. In this table, we give several examples. Recall the change-of-basis have the effect that $\mathbf{B}^\top\mathbb{J}\mathbf{B} = \mathbb{J}_\text{can}$.}
\label{tab:symplectic-bases}
\end{table*}

\begin{theorem}\label{thm:non-canonical-dynamics-darboux}
Let $\Lambda_\Omega$ be a (possibly non-canonical) Poisson structure corresponding to the symplectic structure $\Omega$. Let $\mathbf{F}$ be the change-of-basis matrix for which $\Omega$ assumes the canonical form. Then the canonical dynamics given by Hamilton's equations of motion with Hamiltonian $\tilde{H}(\tilde{\mathbf{q}},\tilde{\mathbf{p}})$ are equivalent to non-canonical Hamiltonian dynamics with Poisson structure $\Lambda_\Omega$. Moreover, $\mathbb{B}_\Omega = \mathbf{B}\mathbb{J}_\text{can}\mathbf{B}^\top$ and $\frac{\mathrm{d}}{\mathrm{d}t}(\tilde{\mathbf{q}},\tilde{\mathbf{p}})^\top = \mathbb{J}_\text{can}\mathbf{B}^\top\nabla_z H(\mathbf{B}(\tilde{\mathbf{q}},\tilde{\mathbf{p}})^\top)$.
\end{theorem}
A proof is given in \cref{app:proof-non-canonical-dynamics-darboux}. 
This result demonstrates that with an appropriate change-of-basis that Hamiltonian dynamics with non-canonical Poisson structures may be reduced to the canonical case for a modified Hamiltonian. 

\subsection{Non-Separability of the Obtained Canonical Form}

A difficulty with such a basis transformation is that the Hamiltonian $\tilde{H}(\tilde{\mathbf{q}},\tilde{\mathbf{p}})$ may not remain separable. Because the leapfrog integrator and its variants either require a separable Hamiltonian to be explicit and become implicit if the Hamiltonian is not separable, are we forced to adopt an implicit integration scheme (such as implicit midpoint) in order to integrate these dynamics? Surprisingly, the answer is no: It is possible to devise a {\it symplectic, symmetric, second-order accurate, and  explicit} numerical integrator for non-separable Hamiltonians. This is the subject of the next section.

\subsection{An explicit integrator for the Non-Separable Form}

We use the technique from \citep{PhysRevE.94.043303}. The key insight is to make a copy of the non-separable Hamiltonian and to integrate both systems simultaneously in a phase-space that is expanded to include additional position variables $\tilde{\mathbf{x}}$ and momentum variables $\tilde{\mathbf{y}}$. Specifically, we integrate the modified Hamiltonian, $\hat{H}(\tilde{\mathbf{q}}, \tilde{\mathbf{p}}, \tilde{\mathbf{x}}, \tilde{\mathbf{y}}) = \tilde{H}(\tilde{\mathbf{q}}, \tilde{\mathbf{y}}) + \tilde{H}(\tilde{\mathbf{x}},\tilde{\mathbf{p}}) + 
    \frac{\omega}{2} \paren{\Vert \tilde{\mathbf{q}}-\tilde{\mathbf{x}}\Vert_2^2 + \Vert \tilde{\mathbf{p}}-\tilde{\mathbf{y}}\Vert_2^2}$.
The parameter $\omega > 0$ controls the degree of binding that enforces $(\tilde{\mathbf{q}}, \tilde{\mathbf{p}})$ to be close to $(\tilde{\mathbf{x}}, \tilde{\mathbf{y}})$. We give the precise integration procedure, based on Strang splitting, for this modified Hamiltonian in \cref{app:explicit-integration-scheme}. Upon integrating the dynamics to time $t>0$, beginning at initial position $(\tilde{\mathbf{q}}_0, \tilde{\mathbf{p}}_0, \tilde{\mathbf{q}}_0, \tilde{\mathbf{p}}_0)$, one finds that both $(\tilde{\mathbf{q}}_t, \tilde{\mathbf{p}}_t)$ or $(\tilde{\mathbf{x}}_t, \tilde{\mathbf{y}}_t)$ is an approximate trajectory for the non-separable Hamiltonian $\tilde{H}(\tilde{\mathbf{q}},\tilde{\mathbf{p}})$ in the symplectic basis. This same integration procedure has already been applied to Riemannian manifold HMC in order to devise an explicit integrator for the non-separable Hamiltonians arising in that setting \citep{rmhmc-explicit}. We note that our method with this explicit integrator satisfies detailed balance only in an asymptotic sense, which is a consequence of the expansion of phase-space used in the integrator; see \cref{app:remarks-on-proposal-operator}.

As the explicit integrator yields an approximate trajectory for these same dynamics by \cref{thm:non-canonical-dynamics-darboux}, we treat it as a substitute for the implicit midpoint integrator in non-canonical Hamiltonian Monte Carlo. \Cref{alg:non-canonical-hmc} contains a practical algorithm for non-canonical HMC using either implicit or explicit integration. {\it We emphasize that the substitution of the explicit integrator for the implicit midpoint algorithm cannot be combined with a Metropolis accept-decision rule to yield a reversible Markov chain. Samples generated by non-canonical HMC with the explicit integrator can only be regarded as approximate, even asymptotically. Therefore, the samples produced by the explicit integrator should be compared against a baseline for consistency.}

\section{Experiments}
\label{sec:experiments}

Experiments were implemented in JAX \citep{jax2018github}. Experiments are carried out in 64-bit precision on CPU. In our experiments evaluating the implicit midpoint rule, we set a convergence tolerance of $1\times 10^{-6}$ or the completion of one-hundred inner fixed point iterations, whichever is satisfied first. Computationally, we note that the explicit integrators for the leapfrog and magnetic leapfrog algorithms require two gradient evaluations per step whereas the explicit integrator for non-canonical dynamics requires four gradient evaluations per step. Assuming that the calculation of the gradient is the most expensive step in HMC, we expect, therefore, that the explicit integrator of non-canonical dynamics will perform approximately half the number of sampling iterations as HMC with leapfrog or magnetic leapfrog integrators within the same time period. Although each individual step can be more expensive, the results below demonstrate that the explicit integrator of non-canonical dynamics can yield advantageous samples and advantageous performance per unit of time.

We consider four variants of Poisson structure in this work. Using the block structure of \cref{eq:poisson-matrix}, we examine the following variants.
\begin{enumerate}
    \item {\bf Canonical} Let $\mathbf{A} = \mathbf{A}^*$ and $\mathbf{E}=\mathbf{G}=\mathbf{0}$. When $\mathbf{A}^*$ is positive semi-definite, this corresponds to mass preconditioning. In the logistic regression experiments we take $\mathbf{A}^*$ as the symmetric square-root of the Fisher information at the mode of the posterior and in the differential equation experiment we take $\mathbf{A}^*=\text{Id}$.
    \item {\bf Magnetic Position} The strategy introduced in \citep{pmlr-v70-tripuraneni17a}, we have $\mathbf{G} = \mathbf{G}^*$, $\mathbf{A} = \mathbf{A}^*$, and $\mathbf{E} = \mathbf{0}$. We generate $\mathbf{G}^*$ by sampling a standard normal matrix and inducing skew-symmetry by $\mathbf{G}^* \mapsto (\mathbf{G}^* - (\mathbf{G}^*)^\top) / k$, where $k\in\mathbb{N}$ is a parameter of the method.
    \item {\bf Magnetic Momentum} Named for its analogy to the magnetic position case, we have $\mathbf{E}=\mathbf{E}^*$, $\mathbf{A} = \mathbf{A}^*$, and $\mathbf{G} = \mathbf{0}$. We generate $\mathbf{E}^*$ in the same way as $\mathbf{G}^*$ in the case of magnetic position structure.
    \item {\bf Coupled Magnet} Take $\mathbf{G} = \mathbf{E} = \mathbf{H}^*$ and $\mathbf{A} = \mathbf{A}^*$. We generate $\mathbf{H}^*$ in the same way as $\mathbf{G}^*$.
\end{enumerate}

\subsection{Logistic Regression Experiments}

To demonstrate non-canonical Hamiltonian Monte Carlo, and to compare it to canonical HMC and the magnetic HMC variant, we consider twelve benchmark logistic regression tasks wherein all features are scaled to have zero mean and unit variance. Our aims are two-fold. First, we will want to assess the performance of the non-canonical dynamics in terms of sampling efficacy. Our measure for this purpose is the effective sample size (ESS); a related concept is the ESS per second, which indicates whether the computational cost associated with the algorithms is justified. We set $k=2$ in these experiments.  We additionally show the potential scale reduction metric \citep{gelman1992}, denoted $\hat{R}$, for these experiments.

\begin{figure*}[ht!]
    \centering
    \begin{subfigure}[b]{\textwidth}
        \includegraphics[width=\textwidth]{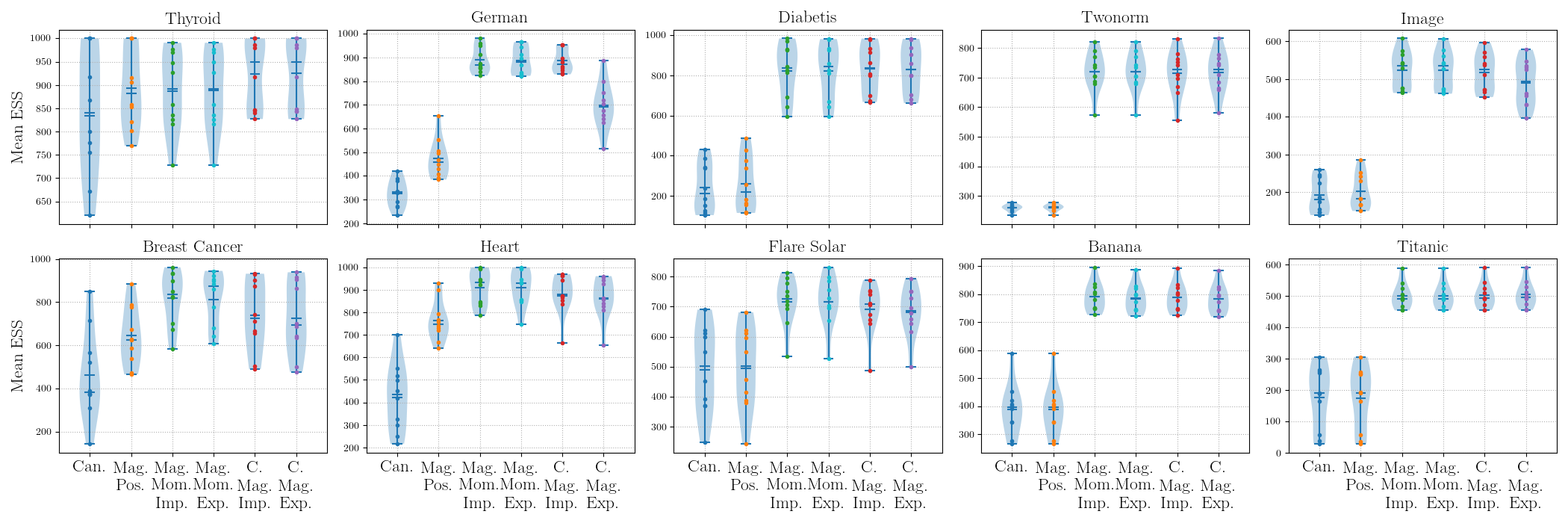}
        \caption{Mean effective sample size}
        \label{subfig:mean}
    \end{subfigure}
    
    \begin{subfigure}[b]{\textwidth}
        \includegraphics[width=\textwidth]{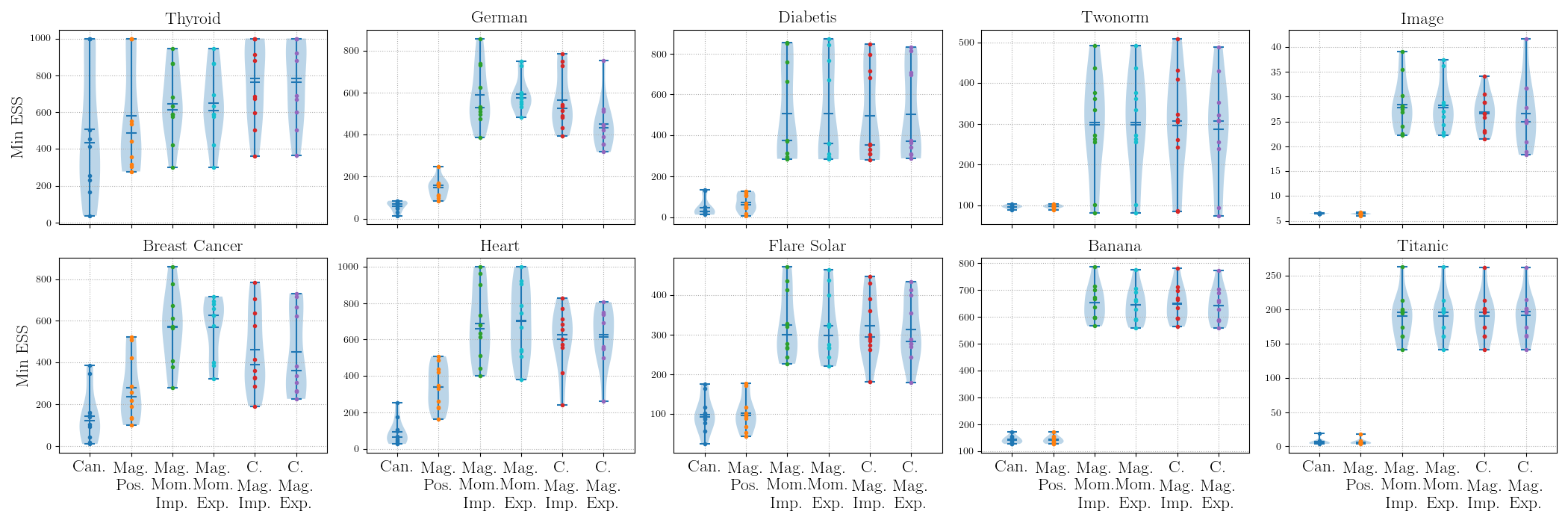}
        \caption{Minimum effective sample size}
        \label{subfig:min}
    \end{subfigure}
    
    \begin{subfigure}[b]{\textwidth}
        \includegraphics[width=\textwidth]{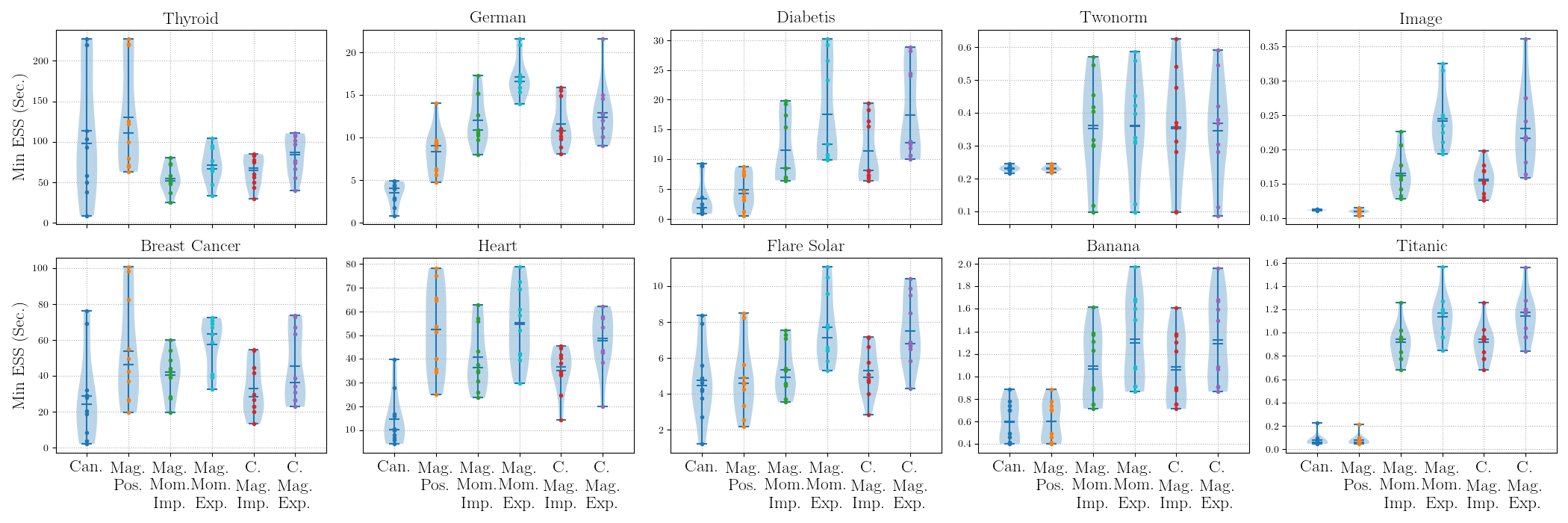}
        \caption{Minimum effective sample size per second}
        \label{subfig:min-per-sec}
    \end{subfigure}
    
    \caption{We evaluate the performance of the non-canonical structure relative to canonical and magnetic variants of Hamiltonian Monte Carlo.  In \cref{subfig:mean}, we find that the average (over all coefficients) estimated effective sample size is significantly larger for HMC procedures with non-canonical structures; the same conclusion holds for minimum effective sample size in \cref{subfig:min}. Moreover, as shown in \cref{subfig:min-per-sec}, in computing the minimum (over all coefficients) effective sample size per second, we also find that the non-canonical dynamics out-perform canonical methods. The methods are abbreviated according to the Poisson structure used as ``Can.'' for canonical, ``Mag. Pos.'' for magnetic position, ``Mag. Mom.'' for magnetic momentum, and ``C. Mag.'' for coupled magnet; where applicable, there is a suffix declaring whether implicit (Imp.) or explicit (Exp.) integration is used.}
    \label{fig:non-canonical-performance}
\end{figure*}
\begin{figure*}[ht!]
    \centering
    \includegraphics[width=\textwidth]{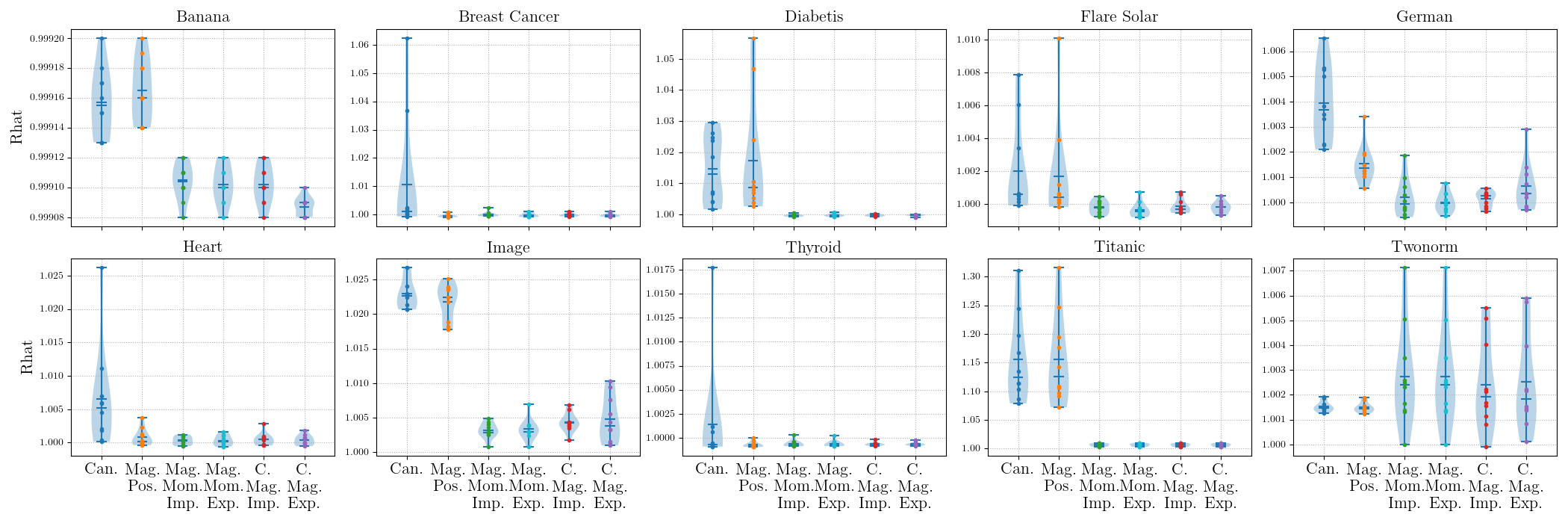}
    \caption{We show the estimated potential scale reduction ($\hat{R}$) for the logistic regression experiments we consider. We find that non-canonical HMC tends to produce smaller $\hat{R}$, although competing methods are below the convergence threshold of 1.05. A notable exception to this is the Titanic dataset, wherein canonical and magnetic HMC struggled to produce convincing convergence.}
    \label{fig:non-canonical-rhat}
\end{figure*}

In our experiments we truncate the effective sample size to the total number of samples. We draw 1,000 samples from a logistic regression posterior with standard normal priors on the coefficients. We integrate Hamiltonian trajectories for one-hundred iteration using a step size given by $1 / (10\cdot n_\text{train})$ where $n_\text{train}$ is the size of the training dataset. The average ESS and the minimum ESS per second are visualized over ten runs of the HMC variants in \cref{fig:non-canonical-performance}. Convergence according to $\hat{R}$ is shown in \cref{fig:non-canonical-rhat}. We find that magnetic momentum performs well in these benchmarks.

\subsection{Fitzhugh-Nagumo Model}

\begin{figure*}[ht!]
    \centering
    \begin{subfigure}[b]{\textwidth}
        \includegraphics[width=\textwidth]{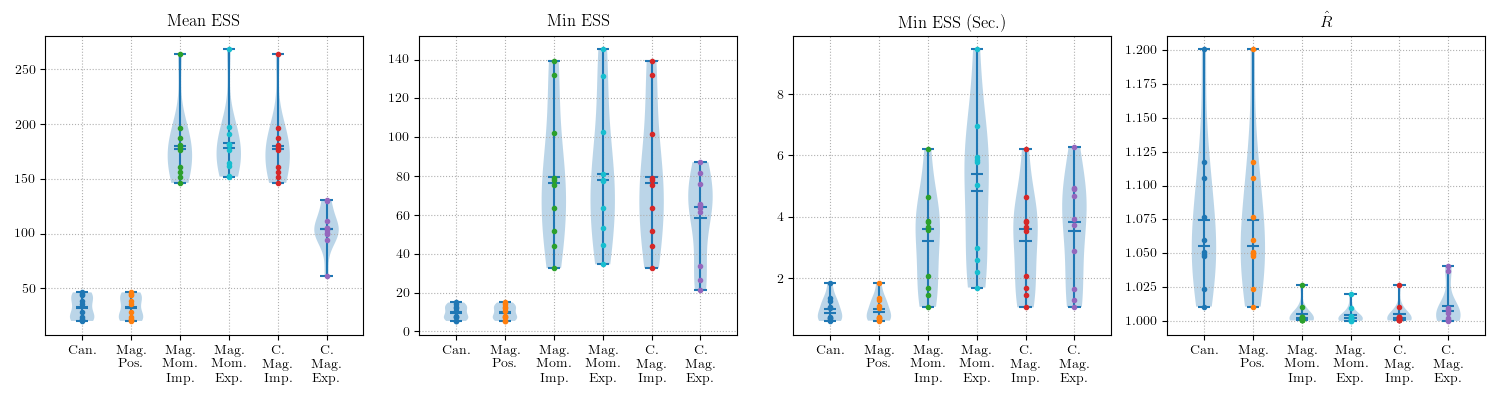}
        \caption{Five steps with step-size $\epsilon=0.005$}
        \label{subfig:fn-a}
    \end{subfigure}
    
    \begin{subfigure}[b]{\textwidth}
        \includegraphics[width=\textwidth]{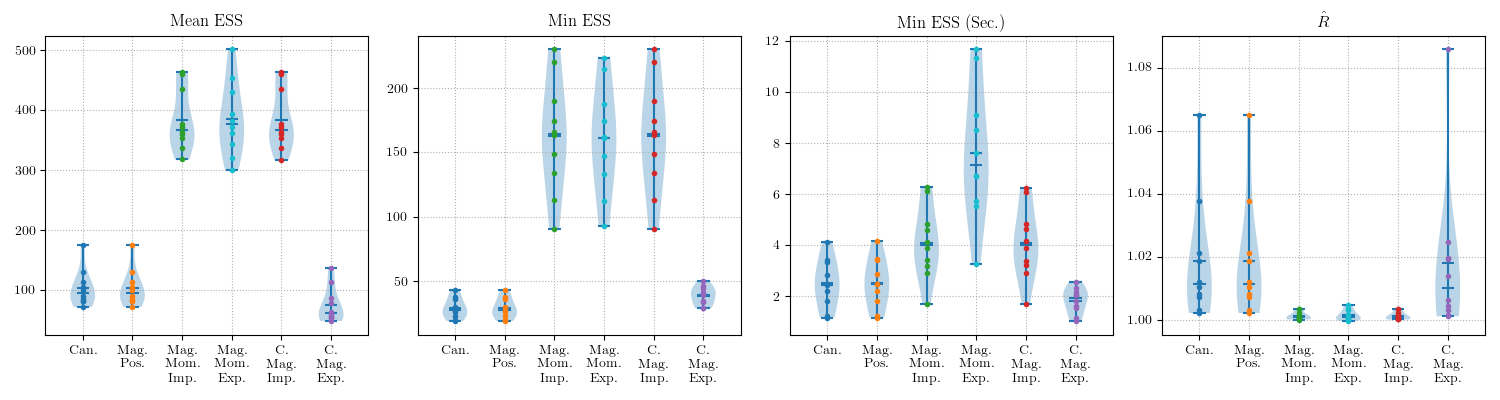}
        \caption{Ten steps with step-size $\epsilon=0.005$}
        \label{subfig:fn-b}
    \end{subfigure}
    
    \begin{subfigure}[b]{\textwidth}
        \includegraphics[width=\textwidth]{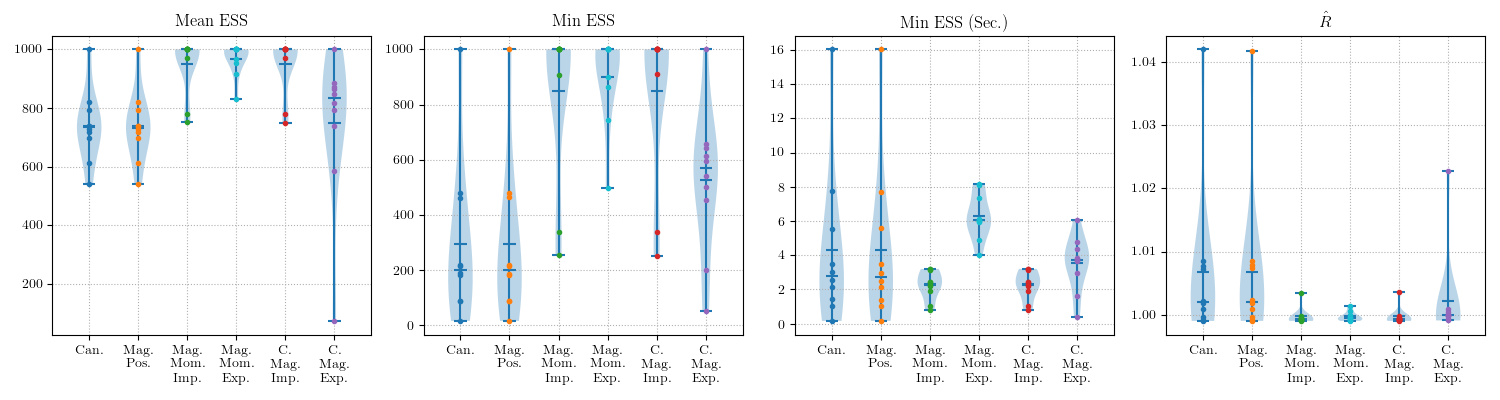}
        \caption{One-hundred steps with step-size $\epsilon=0.005$}
        \label{subfig:fn-c}
    \end{subfigure}
    
    \caption{Performance of non-canonical HMC algorithms on the Fitzhugh-Nagumo model for a fixed step-size of $0.005$ and varying numbers of steps. Non-canonical HMC gives clear improvements in mean ESS and minimum ESS, which the notable exception of the coupled magnet structure with explicit integration. Magnetic momentum with explicit integration is our strongest performing method on this task across the integration periods considered on the minimum ESS per second metric; we note that magnetic momentum with implicit integration is also strong with five and ten integration steps. Non-canonical HMC also exhibits strength in posterior convergence as measured by the $\hat{R}$ statistic (lower is better) indicating superior mixing of the chain into the posterior.}
    \label{fig:fitzhugh-nagumo-performance}
\end{figure*}

We consider inference in a Bayesian model whose likelihood function is computed by solving a non-linear system of differential equations. We specifically consider the Fitzhugh-Nagumo model \citep{pmlr-v70-tripuraneni17a,doi:10.1111/j.1467-9868.2010.00765.x,ijcai2019-791,doi:10.1111/j.1467-9868.2007.00610.x}, whose time derivatives are given by,
\begin{align}
    \dot{V}(t) &= c\cdot(V(t) - V(t)^3 / 3 +R`(t)) \\ \dot{R}(t) &= -(V(t) -a +b\cdot R(t)) / c.
\end{align}
We observe a trajectory of $V(t)$ and $R(t)$ corrupted by Gaussian noise whose true parameters are $a=1/5, b=1/5, c=3$. The noise scale by which $V(t)$ and $R(t)$ are corrupted is set to $\sigma_\text{noise} = 1/10$. We use $\text{Normal}(0, 1/2)$ priors for $a$, $b$, and $c$ and we utilize a Gaussian likelihood with variance $1/2$ over the noise-corrupted observations of $V(t)$ and $R(t)$. Timing comparisons and expected sample size performance are shown in \cref{fig:fitzhugh-nagumo-performance} over ten trials where we generate two-hundred samples from the Fitzhugh-Nagumo model for $t\in[0, 10)$ and integrate the dynamics using a fixed integration step-size of $0.005$. We set $k=50$ in these experiments. Each HMC variant is initialized at the true parameter configuration. We find that non-canonical HMC exhibits better mean ESS and minimum ESS than competing methods and also better convergence as measured by the commonly-used $\hat{R}$ statistic (values closer to one are indicative of better convergence of the chain). Moreover, the explicit integration strategy with the magnetic momentum structure yields superior minimum ESS per second of computation. Notably, explicit integration with the coupled magnet structure struggled on this inference task due to an inability to maintain the Hamiltonian over the integration period resulting in a relatively low rate of transitions.

\begin{table*}[ht!]
\centering
\footnotesize
\begin{tabular}{lccc|ccc|ccc}
\toprule 
& \multicolumn{3}{c}{$a$} & \multicolumn{3}{c}{$b$} & \multicolumn{3}{c}{$c$} \\
& Mean & Std. & Median & Mean & Std. & Median & Mean & Std. & Median \\\bottomrule
Canonical & 0.18 &   0.01 &   0.18  & 0.28 &   0.07 &   0.28 & 2.97 &   0.05 &   2.97 \\
Mag. Pos. & 0.18 &   0.01 &   0.18  &   0.28 &   0.07 &   0.28   & 2.97 &   0.05 &   2.97 \\
Mag. Mom. Imp. & 0.19 &   0.02 &   0.19 & 0.28 &   0.08 &   0.29 & 2.96 &   0.05 &   2.96 \\
Mag. Mom. Exp. & 0.19 &   0.02 &   0.19 & 0.28 &   0.08 &   0.29 & 2.96 &   0.05 &   2.97 \\
C. Mag. Imp. & 0.19 &   0.02 &   0.19 & 0.28 &   0.08 &   0.29 & 2.96 &   0.05 &   2.97 \\
C. Mag. Exp. & 0.18 &   0.02 &   0.19 & 0.28 &   0.10 &   0.28 & 2.96 &   0.06 &   2.97 \\
\bottomrule
\end{tabular}
\caption{Comparison of six coefficient inferences in the Fitzhugh-Nagumo model (which has three parameters $(a,b,c)\in\mathbb{R}^3)$) using canonical HMC, non-canonical HMC with implicit integration (correct MCMC) and explicit integration (approximate MCMC). Results for 100 steps using a step-size of 0.005. The inferences produced by any of the six methods are in approximate agreement.}
\label{tab:fitzhugh-nagumo}
\end{table*}

To provide evidence that the explicit integrators yield approximate samples that are faithful to the target posterior, we report summary statistics for samples generated by the six competing competing structures in \cref{tab:fitzhugh-nagumo}

\section{Conclusion}

This work examines non-canonical Hamiltonian dynamics for use in Hamiltonian Monte Carlo. We give a procedure, based on implicit integration, that yields a provably reversible Markov chain when combined with a Metropolis correction. We further propose an explicit integration scheme for non-canonical dynamics. By using Darboux's Theorem and the symplectic Gram-Schmidt procedure, we construct a coordinate system in which the non-canonical symplectic form assumes the canonical structure. Samples generated by the explicit integrator are not generated from a strictly reversible Markov chain, however, and must be regarded as approximate samples. Nevertheless, we find the samples generated from this procedure faithful to the true posterior and, therefore, the approximation of implicit integration by the explicit method a useful one. We compare these integration strategies with the competing leap-frog integrators for canonical and magnetic position Hamiltonian dynamics on benchmark inference problems. We find that non-canonical symplectic structure leads to more appealing inferences as computed by the mean effective sample size and the minimum effective sample size per second. 

As a direction for future work, it would be desirable to identify symplectic integrators which apply to non-separable Hamiltonians and for which it is not necessary to double the number of gradient evaluations, as the explicit integrator requires. Another important direction is to discover ``recipes'' from which one may prescribe a non-canonical symplectic structure to an inference problem. In our experiments, the choice to generate non-canonical sub-blocks of the Poisson structure by skew-symmetrizing standard normal matrices was an arbitrary one. Such recipes exist for Riemannian manifold HMC in the form of the Fisher information metric; is there a theoretically compelling means to introduce non-canonical symplectic structure as well?

\section*{Broader Impact}

Inference algorithms form the foundation of modern statistical procedures. Markov chain Monte Carlo has established itself as an elegant procedure to draw samples from complex Bayesian posteriors. In this work, we consider the Hamiltonian Monte Carlo (HMC) algorithm, which underlies popular inference libraries such as Stan. The focus of our work is to develop and investigate a generalization of HMC by investigating its underlying geometric foundations, particularly symplectic geometry. Our empirical results suggest that these variants of HMC exhibit better effective sample sizes. The ability to collect random samples which are closer to independent (as opposed to containing high autocorrelation) mean lower variance estimates of quantities of statistical interest, such as mean values. Hence, the potential broader impact of this work is to develop sampling algorithms that generate lower-variance estimates. This will be of interest to fields that utilize Bayesian modeling frameworks such as economics, political science, medicine, and many others. Procedural improvements in these domains will result in more accurate conclusions and better characterizations of uncertainties.

\section*{Acknowledgments}

This material is based upon work supported by the National Science Foundation Graduate Research Fellowship under Grant No. 1752134. Any opinion, findings, and conclusions or recommendations expressed in this material are those of the authors and do not necessarily reflect the views of the National Science Foundation. We thank the Yale Center for Research Computing for helpful advice regarding the computational environment. We thank Marcus Brubaker for insightful discussions and the anonymous reviewers for their helpful comments.

\bibliography{thebib}

\begin{thebibliography}{16}
\providecommand{\natexlab}[1]{#1}
\providecommand{\url}[1]{\texttt{#1}}
\expandafter\ifx\csname urlstyle\endcsname\relax
  \providecommand{\doi}[1]{doi: #1}\else
  \providecommand{\doi}{doi: \begingroup \urlstyle{rm}\Url}\fi

\bibitem[Austin et~al.(1993)Austin, Krishnaprasad, and
  Wang]{Austin1993AlmostPI}
Mark Austin, P.~S. Krishnaprasad, and Li-Sheng Wang.
\newblock Almost poisson integration of rigid body systems.
\newblock \emph{Journal of Computational Physics}, 107:\penalty0 105--117,
  1993.

\bibitem[Bradbury et~al.(2018)Bradbury, Frostig, Hawkins, Johnson, Leary,
  Maclaurin, and Wanderman-Milne]{jax2018github}
James Bradbury, Roy Frostig, Peter Hawkins, Matthew~James Johnson, Chris Leary,
  Dougal Maclaurin, and Skye Wanderman-Milne.
\newblock {JAX}: composable transformations of {P}ython+{N}um{P}y programs,
  2018.
\newblock URL \url{http://github.com/google/jax}.

\bibitem[Brubaker et~al.(2012)Brubaker, Salzmann, and
  Urtasun]{pmlr-v22-brubaker12}
Marcus Brubaker, Mathieu Salzmann, and Raquel Urtasun.
\newblock A family of {MCMC} methods on implicitly defined manifolds.
\newblock In Neil~D. Lawrence and Mark Girolami, editors, \emph{Proceedings of
  the Fifteenth International Conference on Artificial Intelligence and
  Statistics}, volume~22 of \emph{Proceedings of Machine Learning Research},
  pages 161--172, La Palma, Canary Islands, 21--23 Apr 2012. PMLR.
\newblock URL \url{http://proceedings.mlr.press/v22/brubaker12.html}.

\bibitem[Chou and Sankaranarayanan(2019)]{ijcai2019-791}
Yi~Chou and Sriram Sankaranarayanan.
\newblock Bayesian parameter estimation for nonlinear dynamics using
  sensitivity analysis.
\newblock In \emph{Proceedings of the Twenty-Eighth International Joint
  Conference on Artificial Intelligence, {IJCAI-19}}, pages 5708--5714.
  International Joint Conferences on Artificial Intelligence Organization, 7
  2019.
\newblock \doi{10.24963/ijcai.2019/791}.
\newblock URL \url{https://doi.org/10.24963/ijcai.2019/791}.

\bibitem[Cobb et~al.(2019)Cobb, Baydin, Markham, and Roberts]{rmhmc-explicit}
Adam Cobb, Atılım Baydin, Andrew Markham, and Stephen Roberts.
\newblock Introducing an explicit symplectic integration scheme for riemannian
  manifold hamiltonian monte carlo, 10 2019.

\bibitem[Gelman and Rubin(1992)]{gelman1992}
Andrew Gelman and Donald~B. Rubin.
\newblock Inference from iterative simulation using multiple sequences.
\newblock \emph{Statist. Sci.}, 7\penalty0 (4):\penalty0 457--472, 11 1992.
\newblock \doi{10.1214/ss/1177011136}.
\newblock URL \url{https://doi.org/10.1214/ss/1177011136}.

\bibitem[Girolami and Calderhead(2011)]{doi:10.1111/j.1467-9868.2010.00765.x}
Mark Girolami and Ben Calderhead.
\newblock Riemann manifold {L}angevin and {H}amiltonian {M}onte {C}arlo
  methods.
\newblock \emph{Journal of the Royal Statistical Society: Series B (Statistical
  Methodology)}, 73\penalty0 (2):\penalty0 123--214, 2011.
\newblock \doi{10.1111/j.1467-9868.2010.00765.x}.
\newblock URL
  \url{https://rss.onlinelibrary.wiley.com/doi/abs/10.1111/j.1467-9868.2010.00765.x}.

\bibitem[Gol{\'e}(2001)]{gole2001symplectic}
C.~Gol{\'e}.
\newblock \emph{Symplectic Twist Maps: Global Variational Techniques}.
\newblock Advanced series in nonlinear dynamics. World Scientific, 2001.
\newblock ISBN 9789812810762.
\newblock URL \url{https://books.google.com/books?id=qhni\_1MrvkQC}.

\bibitem[Livingstone et~al.(2016)Livingstone, Betancourt, Byrne, and
  Girolami]{livingstone2016geometric}
Samuel Livingstone, Michael Betancourt, Simon Byrne, and Mark Girolami.
\newblock On the geometric ergodicity of {H}amiltonian {M}onte {C}arlo, 2016.

\bibitem[Marsden and Ratiu(2002)]{marsden2002introduction}
J.E. Marsden and T.~Ratiu.
\newblock \emph{Introduction to Mechanics and Symmetry: A Basic Exposition of
  Classical Mechanical Systems}.
\newblock Texts in Applied Mathematics. Springer New York, 2002.
\newblock ISBN 9780387986432.
\newblock URL \url{https://books.google.com/books?id=I2gH9ZIs-3AC}.

\bibitem[Marsden(1999)]{Marsden1999ParkCL}
Jerrold~E. Marsden.
\newblock Park city lectures on mechanics, dynamics, and symmetry, 1999.

\bibitem[Neal(1995)]{10.5555/922680}
Radford~M. Neal.
\newblock \emph{Bayesian Learning for Neural Networks}.
\newblock PhD thesis, University of Toronto, CAN, 1995.
\newblock AAINN02676.

\bibitem[Neal(2010)]{1206.1901}
Radford~M. Neal.
\newblock {MCMC} using {Hamiltonian} dynamics.
\newblock \emph{Handbook of Markov Chain Monte Carlo}, 54:\penalty0 113--162,
  2010.

\bibitem[Ramsay et~al.(2007)Ramsay, Hooker, Campbell, and
  Cao]{doi:10.1111/j.1467-9868.2007.00610.x}
J.~O. Ramsay, G.~Hooker, D.~Campbell, and J.~Cao.
\newblock Parameter estimation for differential equations: a generalized
  smoothing approach.
\newblock \emph{Journal of the Royal Statistical Society: Series B (Statistical
  Methodology)}, 69\penalty0 (5):\penalty0 741--796, 2007.
\newblock \doi{10.1111/j.1467-9868.2007.00610.x}.
\newblock URL
  \url{https://rss.onlinelibrary.wiley.com/doi/abs/10.1111/j.1467-9868.2007.00610.x}.

\bibitem[Tao(2016)]{PhysRevE.94.043303}
Molei Tao.
\newblock {Explicit symplectic approximation of nonseparable Hamiltonians:
  algorithm and long time performance}.
\newblock \emph{Phys. Rev. E}, 94:\penalty0 043303, Oct 2016.
\newblock \doi{10.1103/PhysRevE.94.043303}.
\newblock URL \url{https://link.aps.org/doi/10.1103/PhysRevE.94.043303}.

\bibitem[Tripuraneni et~al.(2017)Tripuraneni, Rowland, Ghahramani, and
  Turner]{pmlr-v70-tripuraneni17a}
Nilesh Tripuraneni, Mark Rowland, Zoubin Ghahramani, and Richard Turner.
\newblock Magnetic {H}amiltonian {M}onte {C}arlo.
\newblock In Doina Precup and Yee~Whye Teh, editors, \emph{Proceedings of the
  34th International Conference on Machine Learning}, volume~70 of
  \emph{Proceedings of Machine Learning Research}, pages 3453--3461,
  International Convention Centre, Sydney, Australia, 06--11 Aug 2017. PMLR.
\newblock URL \url{http://proceedings.mlr.press/v70/tripuraneni17a.html}.

\end{thebibliography}

\newpage
\onecolumn
\appendix

\section{Hamiltonian Monte Carlo}\label{app:hamiltonian-monte-carlo}

For completeness we briefly review the relationship between Markov chain Monte Carlo and the numerical integration of Hamiltonian dynamics. Suppose we are given a symplectic integrator (such as leapfrog or implicit midpoint as discussed in \cref{app:numerical-integration}) and a potential function $U(\mathbf{q})$ which defines a density $p(\mathbf{q})\propto \exp(-U(\mathbf{q}))$. To generate samples from $p(\mathbf{q})$, HMC defines a separable Hamiltonian
\begin{align}
    H(\mathbf{q},\mathbf{p}) = U(\mathbf{q}) + \frac{1}{2}\mathbf{p}^\top \mathbf{p}.
\end{align}
By considering the distribution $p(\mathbf{q},\mathbf{p}) \propto\exp(-H(\mathbf{q},\mathbf{p}))$ in phase-space, one observes two properties: (i) the marginal over $\mathbf{p}$ of $p(\mathbf{q},\mathbf{p})$ is $p(\mathbf{q})$ and (ii) the marginal distribution over $\mathbf{q}$ of $p(\mathbf{q},\mathbf{p})$ is standard normal. Given an initial position $\mathbf{q}$ and sampling $\mathbf{p}$ from its marginal distribution, we applying the symplectic integrator for some number of steps to yield a point in phase-space $(\mathbf{q}',\mathbf{p}')$ which is then mapped to the candidate proposal $(\mathbf{q}',-\mathbf{p}')$ (this sign flip is necessary to ensure reversibility of the proposal distribution). Exploiting the volume preservation and symmetry of the integrator, one can demonstrate that accepting or rejecting the candidate according to a Metropolis-Hastings step yields a Markov chain satisfying detailed balance in phase-space for the distribution $p(\mathbf{q},\mathbf{p})$. Iterating this procedure, and ultimately projecting samples to the position variable $\mathbf{q}$ alone, yields samples from the target density $p(\mathbf{q})\propto \exp(-U(\mathbf{q}))$.

\section{Numerical Integration}\label{app:numerical-integration}

An important, and possibly surprising, fact about discretizations of Hamilton's equations of motion is the existence of integrators that preserve properties (ii) and (iii) of the continuous-time dynamics. Such integrators are called symplectic integrators, meaning that the discrete flow they generate is a symplectic map. If an integrator also preserves energy then it is an exact solution to Hamilton's equations of motion; see \citep{Marsden1999ParkCL}.

\begin{definition}[Poisson Bracket]
Let $V$ be a vector space and let $Z = V\times V^*$ be equipped with a symplectic structure $\Omega$. Given two functions $F,G$ on $Z$, we defined the Poisson bracket of $F$ and $G$ to be the function
\begin{align}
    \set{F, G}(z) &= \Omega(X_F(z), X_G(z)) \\
    &= \mathbf{D} F(z)\cdot X_G(z) \\
    &= \mathbf{D} F(z)\cdot \mathbb{B} ~\mathbf{D} G(z)
\end{align}
where $X_F$ ($X_G$, respectively) is the Hamiltonian vector field corresponding to $F$ ($G$, resp.).
\end{definition}

\begin{definition}[Symplectic Diffeomorphism and Poisson Automorphism \citep{Austin1993AlmostPI,marsden2002introduction}]
A diffeomorphism $\Phi$ is a Poisson automorphism if it preserves the Poisson structure:
\begin{align}
    (\set{F, G}\circ \Phi)(z) = \set{(F\circ \Phi)(z), (G\circ \Phi)(z)}.
\end{align}
A diffeomorphism $\Phi$ is called symplectic if
\begin{align}
    \Omega(v, w) = \Omega(\mathbf{D}\Phi(z)\cdot v, \mathbf{D}\Phi(z) \cdot w).
\end{align}
for all $v, w, z\in Z$. A diffeomorphism is symplectic if and only if it is a Poisson automorphism.
\end{definition}

We will express the discrete flow of an integrator with step-size $\epsilon$ by $z_{n+1} = \Phi_\epsilon(z_n)$

The most popular implementation of a symplectic integrator for MCMC purposes is the leapfrog integrator. Assuming a separable Hamiltonian with quadratic kinetic energy, a vector space with canonical symplectic structure, and an integration step-size $\epsilon > 0$, the leapfrog integrator updates position and momentum variables according to $(\mathbf{q}_0, \mathbf{p}_0)\to (\mathbf{q}_1, \mathbf{p}_1)$ as
\begin{align}
    \mathbf{p}_{1/2} &= \mathbf{p}_0 - \frac{\epsilon}{2} \nabla_\mathbf{q} U(\mathbf{q}_0) \\
    \mathbf{q}_1 &= \mathbf{q}_0 + \epsilon ~\mathbf{p}_{1/2} \\
    \mathbf{p}_{1} &= \mathbf{p}_{1/2} - \frac{\epsilon}{2} \nabla_\mathbf{q} U(\mathbf{q}_1).
\end{align}
Generalizations of the leapfrog algorithm for HMC exist that can handle non-separable Hamiltonians \citep{doi:10.1111/j.1467-9868.2010.00765.x} and holonomic manifold constraints exist \citep{pmlr-v22-brubaker12}. However, these assume the canonical symplectic structure.

The leapfrog rule is not the only choice for a symplectic integration procedure, however. Another algorithm, known as the implicit midpoint rule, is an integration procedure we will encounter in this work.
\begin{definition}[Implicit Midpoint Integrator]
Let $X_H(z)$ be a Hamiltonian vector field corresponding to Hamiltonian $H$. The implicit midpoint integrator is defined by the (implicit) update
\begin{align}
    z_{n+1} = z_n + \epsilon\cdot X_H\paren{\frac{z_{n+1} + z_n}{2}}
\end{align}
where $\epsilon > 0$ is the integration step-size.
\end{definition}
In practice, the implicit midpoint integrator can be implemented by fixed-point iteration. Integrators are called symmetric if their discrete flow satisfies $\Phi_\epsilon\circ \Phi_{-\epsilon} = \text{Id}$. Integrators are called $k^\text{th}$-order accurate if for all $l\leq k$,
\begin{align}
    \frac{\mathrm{d}^l}{\d{\epsilon^l}} \Phi_{\epsilon}(z)\big|_{\epsilon=0} =     \frac{\mathrm{d}^l}{\d{t^l}} z\big|_{t=0}.
\end{align}
Both the leapfrog integrator and the implicit midpoint rule are symmetric, symplectic and second-order accurate for canonical dynamics. Remarkably, in the case of the midpoint rule, more can be said. The following result is from \citep{Marsden1999ParkCL,Austin1993AlmostPI}.
\begin{theorem}
The implicit midpoint integrator is symmetric, symplectic, and second-order accurate. These properties hold for non-canonical Hamiltonian dynamics with constant Poisson structure.
\end{theorem}

\section{Proof of \Cref{thm:time-reversal-poisson}}\label{app:proof-time-reversal-poisson}

We want to generalize Lemma 2 in the Magnetic HMC paper \citep{pmlr-v70-tripuraneni17a} to the case of positionally-varying sympletic structure. Let's begin with a proof of the usual reversibility of Hamiltonian dynamics to get a flavor for the argument. Recall that Hamilton's (canonical) equations of motion are,
\begin{align}
    \dot{\mathbf{q}} = \nabla_\mathbf{p} H(\mathbf{q}, \mathbf{p}) ~~~~~~~~~~~ \dot{p} = -\nabla_\mathbf{q} H(\mathbf{q}, \mathbf{p})
\end{align}
We were to consider the negation of $p$ these equations would read:
\begin{align}
    \dot{\mathbf{q}} = \nabla_\mathbf{p} H(\mathbf{q}, -\mathbf{p}) ~~~~~~~~~~~ -\dot{p} = -\nabla_\mathbf{q} H(\mathbf{q}, -\mathbf{p})
\end{align}
Let $\tilde{\mathbf{p}} = -\mathbf{p}$. Then the evolution of the negated momentum satisfies $\frac{\d{}}{\d{t}}\tilde{\mathbf{p}} = -\nabla_\mathbf{q} H(\mathbf{q}, \tilde{\mathbf{p}})$. The evolution of state satisfies,
\begin{align}
    \dot{\mathbf{q}} &= \nabla_\mathbf{p} H(\mathbf{q}, -\mathbf{p}) \\
    &= -\nabla_{-p} H(\mathbf{q}, -\mathbf{p}) \\
    &= -\nabla_{\tilde{p}} H(\mathbf{q}, \tilde{\mathbf{p}})
\end{align}
Hence $\frac{\d{}}{\d{t}}\tilde{\mathbf{p}} = -\nabla_\mathbf{q} H(\mathbf{q}, \tilde{\mathbf{p}})$. Now suppose that the Hamiltonian has a form such that $H(\mathbf{q}, \mathbf{p}) = H(\mathbf{q}, -\mathbf{p})$, as is typically the case for quadratic kinetic energy components. This yields $\dot{\mathbf{q}} = -\nabla_\mathbf{p} H(\mathbf{q}, \mathbf{p})$. Then by identifying the negative time derivative as the derivative in negative time, we can view these equations as describing behavior of the system when time is run in reverse:
\begin{align}
    \frac{\d{}}{\d{(-t)}} \mathbf{q} &= \nabla_{\tilde{\mathbf{p}}} H(\mathbf{q}, \tilde{\mathbf{p}}) = \nabla_\mathbf{p} H(\mathbf{q}, \mathbf{p}) \\
    \frac{\d{}}{\d{(-t)}} \mathbf{p} &= -\nabla_\mathbf{q} H(\mathbf{q}, \tilde{\mathbf{p}}) = -\nabla_\mathbf{q} H(\mathbf{q}, \mathbf{p})
\end{align}

We want a generalization of Lemma 2 that allows for state-dependent symplectic structure. Let's postulate the following dynamics:
\begin{align}\label{eq:original-hamiltonian}
    \begin{bmatrix} \dot{\mathbf{q}} \\ \dot{\mathbf{p}} \end{bmatrix} &= \begin{bmatrix} \mathbf{E}(\mathbf{q}, \mathbf{p}) & \mathbf{F}(\mathbf{q}, \mathbf{p}) \\ -\mathbf{F}^\top(\mathbf{q}, \mathbf{p}) & \mathbf{G}(\mathbf{q}, \mathbf{p}) \end{bmatrix} \begin{bmatrix} \nabla_\mathbf{q} H(\mathbf{q}, \mathbf{p}) \\\nabla_\mathbf{p} H(\mathbf{q}, \mathbf{p})\end{bmatrix}
\end{align}
If we consider the negation of $\mathbf{p}$ we obtain the dynamics:
\begin{align}
    \begin{bmatrix} \dot{\mathbf{q}} \\ -\dot{\mathbf{p}} \end{bmatrix} &= \begin{bmatrix} \mathbf{E}(\mathbf{q}, -\mathbf{p}) & \mathbf{F}(\mathbf{q}, -\mathbf{p}) \\ -\mathbf{F}^\top(\mathbf{q}, -\mathbf{p}) & \mathbf{G}(\mathbf{q}, -\mathbf{p}) \end{bmatrix} \begin{bmatrix} \nabla_\mathbf{q} H(\mathbf{q}, -\mathbf{p}) \\\nabla_\mathbf{p} H(\mathbf{q}, -\mathbf{p})\end{bmatrix} \\
    &= \begin{bmatrix}
        \mathbf{E}(\mathbf{q}, -\mathbf{p}) \nabla_\mathbf{q} H(\mathbf{q}, -\mathbf{p})  + \mathbf{F}(\mathbf{q}, -\mathbf{p})\nabla_\mathbf{p} H(\mathbf{q}, -\mathbf{p}) \\
        -\mathbf{F}^\top(\mathbf{q}, -\mathbf{p})\nabla_\mathbf{q} H(\mathbf{q}, -\mathbf{p}) + \mathbf{G}(\mathbf{q}, -\mathbf{p})\nabla_\mathbf{p} H(\mathbf{q}, -\mathbf{p})
    \end{bmatrix} \\
    &= \begin{bmatrix}
        \mathbf{E}(\mathbf{q}, -\mathbf{p}) \nabla_\mathbf{q} H(\mathbf{q}, \mathbf{p}) - \mathbf{F}(\mathbf{q}, -\mathbf{p})\nabla_\mathbf{p} H(\mathbf{q}, \mathbf{p}) \\
        \mathbf{F}^\top(\mathbf{q}, -\mathbf{p})\nabla_\mathbf{q} H(\mathbf{q}, \mathbf{p}) - \mathbf{G}(\mathbf{q}, -\mathbf{p})\nabla_\mathbf{p} H(\mathbf{q}, \mathbf{p})
    \end{bmatrix} \\
    \begin{bmatrix} -\dot{\mathbf{q}} \\ -\dot{\mathbf{p}} \end{bmatrix} &= \begin{bmatrix}
        -\mathbf{E}(\mathbf{q}, -\mathbf{p}) \nabla_\mathbf{q} H(\mathbf{q}, \mathbf{p}) + \mathbf{F}(\mathbf{q}, -\mathbf{p})\nabla_\mathbf{p} H(\mathbf{q}, \mathbf{p}) \\
        \mathbf{F}^\top(\mathbf{q}, -\mathbf{p})\nabla_\mathbf{q} H(\mathbf{q}, \mathbf{p}) - \mathbf{G}(\mathbf{q}, -\mathbf{p})\nabla_\mathbf{p} H(\mathbf{q}, \mathbf{p})
    \end{bmatrix} \\
    &= \begin{bmatrix} -\mathbf{E}(\mathbf{q}, -\mathbf{p}) & \mathbf{F}(\mathbf{q}, -\mathbf{p}) \\ -\mathbf{F}^\top(\mathbf{q}, -\mathbf{p}) & -\mathbf{G}(\mathbf{q}, -\mathbf{p}) \end{bmatrix} \begin{bmatrix} \nabla_\mathbf{q} H(\mathbf{q}, \mathbf{p}) \\ \nabla_\mathbf{p} H(\mathbf{q}, \mathbf{p}) \end{bmatrix} \label{eq:reverse-hamiltonian}
\end{align}
Making again the identification of the negative time derivative with the derivative in negative time, we find the reversed dynamics satisfy the natural state-dependent symplectic generalization of Lemma 2. Here is the physical interpretation of this result: Let $(\mathbf{q}(t), \mathbf{p}(t))$ be the phase-space position at time $t$ if the system evolves according to the Hamiltonian in \cref{eq:original-hamiltonian}. Suppose we terminate the dynamics at a fixed time $\tau$ and consider the reversed trajectory $(\tilde{\mathbf{q}}(t), \tilde{\mathbf{p}}(t)) = (\mathbf{q}(\tau - t), -\mathbf{p}(\tau - t))$. Then $(\tilde{\mathbf{q}}(t), \tilde{\mathbf{p}}(t))$ is a solution to the ``reversed'' dynamics in \cref{eq:reverse-hamiltonian} and at time $\tau$ we will have $(\tilde{\mathbf{q}}(\tau), \tilde{\mathbf{p}}(\tau)) = (\mathbf{q}(0), -\mathbf{p}(0))$.

\section{Proof of \Cref{thm:non-canonical-dynamics-darboux}}\label{app:proof-non-canonical-dynamics-darboux}

{\bf Theorem 3.} Let $\Lambda$ be a (possibly non-canonical) Poisson structure corresponding to the symplectic structure $\Omega$. Let $\mathbf{F}$ be the change-of-basis matrix for which $\Omega$ assumes the canonical form. Then the canonical dynamics given by Hamilton's equations of motion with Hamiltonian $\tilde{H}(\tilde{\mathbf{q}},\tilde{\mathbf{p}})$ are equivalent to non-canonical Hamiltonian dynamics with Poisson structure $\Lambda$. Moreover, $\mathbb{B} = \mathbf{B}\mathbb{J}_\text{can}\mathbf{B}^\top$ and $\frac{\mathrm{d}}{\mathrm{d}t}(\tilde{\mathbf{q}},\tilde{\mathbf{p}})^\top = \mathbb{J}_\text{can}\mathbf{B}^\top\nabla_z H(\mathbf{B}(\tilde{\mathbf{q}},\tilde{\mathbf{p}})^\top)$.

To prove this theorem, we will first require the following lemma.
\begin{lemma}
$\mathbb{B} = \mathbf{B}\mathbb{J}_\text{can}\mathbf{B}^\top$
\end{lemma}
\begin{proof}
Recall the fundamental property of the Darboux basis: $\mathbb{J} = \mathbf{F}^\top \mathbb{J}_\text{can}\mathbf{F}$. The matrix of the Poisson structure is related to the matrix of the symplectic structure according to $\mathbb{B}=(-\mathbb{J})^{-1}$. Recall $\mathbf{B}=\mathbf{F}^{-1}$.
We have,
\begin{align}
    (-\mathbb{J})^{-1} &= (\mathbf{F}^\top(-\mathbb{J}_\text{can})\mathbf{F})^{-1} \\
    &= \mathbf{F}^{-1} (-\mathbb{J}_\text{can})^{-1} (\mathbf{F}^\top)^{-1} \\
    &= \mathbf{B} \mathbb{J}_\text{can} \mathbf{B}^\top.
\end{align}
It is easily verified that $(-\mathbb{J}_\text{can})^{-1} = \mathbb{J}_\text{can}$.
\end{proof}

\begin{proof}
Let $\tilde{z} = \mathbf{F}z$ and recall $\mathbf{B}=\mathbf{F}^{-1}$. In canonical coordinates, Hamilton's equations of motion state that the time evolution of the particle will obey
\begin{align}
    \frac{\mathrm{d}}{\mathrm{d}t} \tilde{z} = \mathbb{J}_\text{can} ~\mathbf{D}\tilde{H}(\tilde{z})
\end{align}
By the chain rule $\mathbf{D}\tilde{H}(\tilde{z}) = \mathbf{D} H(\mathbf{B}\tilde{z}) = \mathbf{B}^\top \nabla_z H(z)$ where we have used that $z = \mathbf{B}\tilde{z}$. Hence,
\begin{align}
    \frac{\mathrm{d}}{\mathrm{d}t} \tilde{z} = \mathbb{J}_\text{can}\mathbf{B}^\top \nabla_z H(z)
\end{align}
Now by substitution via $\tilde{z} = \mathbf{F}z$ we find
\begin{align}
    &\frac{\mathrm{d}}{\mathrm{d}t} \mathbf{F}z = \mathbb{J}_\text{can}\mathbf{B}^\top \nabla_z H(z) \\
    \implies& \frac{\mathrm{d}}{\mathrm{d}t} z = \mathbf{B}\mathbb{J}_\text{can}\mathbf{B}^\top \nabla_z H(z)
\end{align}
whereupon the identification (by non-degeneracy of the symplectic form) of $\mathbb{B} = \mathbf{B}\mathbb{J}_\text{can}\mathbf{B}^\top$ shows that the particle evolution in canonical coordinates is identical to the motion in non-canonical coordinates; that is, we have indeed shown $\dot{z} = \mathbb{B}~\mathbf{D}H(z)$.
\end{proof}

\section{Symplectic Gram-Schmidt}\label{app:symplectic-gram-schmidt}

The symplectic Gram-Schmidt procedure we use in \cref{alg:symplectic-gram-schmidt} is modified from \citep{gole2001symplectic}.

\begin{algorithm}[ht!]
\caption{A modified version of the symplectic Gram-Schmidt procedure such that, at termination, it is evident how the basis vectors must be rearranged in order to partition the canonical coordinates into state and momentum variables.}
\label{alg:symplectic-gram-schmidt}
\begin{algorithmic}[1]
\State \textbf{Input}: Symplectic form $\Omega$ with matrix $\mathbb{J}\in\text{Skew}(2n)$ in non-canonical coordinates.
\State \textbf{Output}: Basis $\mathbf{B}$ in which the matrix of $\Omega$ assumes canonical form.
\State $\mathbf{B} = \varnothing$
\For{$i = 1,\ldots, n$}
    \State $w, v\overset{\text{i.i.d.}}{\sim}\text{Normal}(\mathbf{0}, \text{Id})$ such that $w, v \in\R^{2n}$.
    \If{$\mathbf{B}$ is not empty}
        \State Project $v$ and $w$ into the orthogonal complement of the subspace spanned by the columns of $\mathbb{J}\mathbf{B}$.
    \EndIf
    \State $O = \Omega(v, w) = v^\top\mathbb{J}w$. Compute $v' = (v\cdot \text{sign}(O)) / \sqrt{\vert O\vert}$ and $w' = w / \sqrt{\vert O\vert}$.
    \State Add $v'$ and $w'$ as columns of $\mathbf{B}$.
\EndFor    
\State Shuffle the columns of $\mathbf{B}$ (using zero-based indexing) according to the permutation $\set{0, 2, \ldots, 2n-2, 1, 3, \ldots, 2n-1}$.
\State \textbf{Return}: $\mathbf{B}$.
\end{algorithmic}
\end{algorithm}

We note that the algorithm is randomized and the output is not unique. In practice one may run symplectic Gram-Schmidt multiple times and choose and choose a basis with the smallest Frobenius norm.

\section{Explicit Integration Scheme}\label{app:explicit-integration-scheme}

This integration strategy is due to \citep{PhysRevE.94.043303}. The objective is to integrate Hamiltonian dynamics whose behavior in canonical coordinates is given by a non-separable Hamiltonian $H(\tilde{\mathbf{q}}, \tilde{\mathbf{p}})$. The idea to create an augmented phase-space with additional position variables $\tilde{\mathbf{x}}$ and momentum variables $\tilde{\mathbf{y}}$ and to define an expanded Hamiltonian in the expanded phase-space:
\begin{align}\label{eq:expanded-hamiltonian}
    \hat{H}(\tilde{\mathbf{q}}, \tilde{\mathbf{p}}, \tilde{\mathbf{x}}, \tilde{\mathbf{y}}) = H(\tilde{\mathbf{q}},\tilde{\mathbf{y}}) + H(\tilde{\mathbf{x}}, \tilde{\mathbf{p}}) + \frac{\omega}{2}\paren{\Vert \tilde{\mathbf{q}}-\tilde{\mathbf{x}}\Vert_2^2 + \Vert \tilde{\mathbf{p}}-\tilde{\mathbf{y}}\Vert_2^2}
\end{align}
where $\omega>0$ is the binding term that encourages $\tilde{\mathbf{q}}$ and $\tilde{\mathbf{x}}$ to be close, as well as $\tilde{\mathbf{p}}$ and $\tilde{\mathbf{y}}$ to be close. This binding parameter is a hyperparameter of the method. By splitting this Hamiltonian and considering constituent flows for a discretization of time $\epsilon > 0$, \citep{PhysRevE.94.043303} defines the following component integrators:
\begin{align}
    \Phi_1^{\epsilon} : (\tilde{\mathbf{q}}, \tilde{\mathbf{p}}, \tilde{\mathbf{x}}, \tilde{\mathbf{y}})^\top \to (\tilde{\mathbf{q}}, \tilde{\mathbf{p}} - \epsilon\nabla_{\tilde{\mathbf{q}}} H(\tilde{\mathbf{q}}, \tilde{\mathbf{y}}), \tilde{\mathbf{x}} + \epsilon\nabla_{\tilde{\mathbf{y}}} H(\tilde{\mathbf{x}}, \tilde{\mathbf{y}}), \tilde{\mathbf{y}})^\top \label{eq:integrator-one} \\
    \Phi_2^{\epsilon} : (\tilde{\mathbf{q}}, \tilde{\mathbf{p}}, \tilde{\mathbf{x}}, \tilde{\mathbf{y}})^\top \to (\tilde{\mathbf{q}} + \epsilon\nabla_{\tilde{\mathbf{p}}} H(\tilde{\mathbf{x}}, \tilde{\mathbf{p}}), \tilde{\mathbf{p}}, \tilde{\mathbf{x}}, \tilde{\mathbf{y}} - \epsilon \nabla_{\tilde{\mathbf{x}}} H(\tilde{\mathbf{x}}, \tilde{\mathbf{p}}))^\top \label{eq:integrator-two} \\
    \Phi_3^{\epsilon, \omega}: (\tilde{\mathbf{q}}, \tilde{\mathbf{p}}, \tilde{\mathbf{x}}, \tilde{\mathbf{y}})^\top \to \frac{1}{2}\begin{pmatrix} \begin{pmatrix} \tilde{\mathbf{q}} + \tilde{\mathbf{x}} \\ \tilde{\mathbf{p}} + \tilde{\mathbf{y}} \end{pmatrix} + \mathbf{R}(\omega,\delta) \begin{pmatrix} \tilde{\mathbf{q}}-\tilde{\mathbf{x}} \\\tilde{\mathbf{p}}-\tilde{\mathbf{y}} \end{pmatrix} \\ \begin{pmatrix} \tilde{\mathbf{q}} + \tilde{\mathbf{x}} \\ \tilde{\mathbf{p}} + \tilde{\mathbf{y}} \end{pmatrix} - \mathbf{R}(\omega,\delta) \begin{pmatrix} \tilde{\mathbf{q}}-\tilde{\mathbf{x}} \\\tilde{\mathbf{p}}-\tilde{\mathbf{y}} \end{pmatrix}
    \end{pmatrix}\label{eq:integrator-three}
\end{align}
where
\begin{align}
    \mathbf{R}(\omega,\delta) = \begin{pmatrix} \cos(2\epsilon \omega)~\text{Id} & \sin(2\epsilon \omega)~\text{Id} \\
    -\sin(2\epsilon \omega)~\text{Id} & \cos(2\epsilon \omega)~\text{Id}\end{pmatrix}.
\end{align}
The symplectic, symmetric, second-order accurate integrator is then given by the composition of these flows as
\begin{align}
    \Phi^{\epsilon} = \Phi^{\epsilon/2}_1 \circ \Phi^{\epsilon/2}_2 \circ \Phi^{\epsilon, \omega}_3 \circ \Phi^{\epsilon/2}_2  \circ \Phi^{\epsilon/2}_1.
\end{align}
Applying Hamilton's canonical equations of motion to the Hamiltonian $\hat{H}(\tilde{\mathbf{q}},\tilde{\mathbf{p}},\tilde{\mathbf{x}},\tilde{\mathbf{y}})$ yields the following system of differential equations.
\begin{align}
    \frac{\d{}}{\d{t}} \tilde{\mathbf{q}} &= \nabla_{\tilde{\mathbf{p}}} H(\tilde{\mathbf{x}},\tilde{\mathbf{p}}) + \omega(\tilde{\mathbf{p}}-\tilde{\mathbf{y}}) \\
    \frac{\d{}}{\d{t}} \tilde{\mathbf{p}} &= -\nabla_{\tilde{\mathbf{q}}} H(\tilde{\mathbf{q}},\tilde{\mathbf{y}}) - \omega(\tilde{\mathbf{q}}-\tilde{\mathbf{x}}) \\
    \frac{\d{}}{\d{t}} \tilde{\mathbf{x}} &= \nabla_{\tilde{\mathbf{y}}} H(\tilde{\mathbf{q}},\tilde{\mathbf{y}}) + \omega(\tilde{\mathbf{y}}-\tilde{\mathbf{p}}) \\
    \frac{\d{}}{\d{t}} \tilde{\mathbf{y}} &= -\nabla_{\tilde{\mathbf{x}}} H(\tilde{\mathbf{x}},\tilde{\mathbf{p}}) - \omega(\tilde{\mathbf{x}}-\tilde{\mathbf{q}}) \\
\end{align}
whereupon the technique of Strang splitting yields the integrators in \cref{eq:integrator-one}, \cref{eq:integrator-two}, and \cref{eq:integrator-three}.

In our discussion of detailed balance we will require the following result from \citep{PhysRevE.94.043303}.
\begin{theorem}\label{thm:solution-binding}
Let $(\tilde{\tilde{\mathbf{q}}}', \tilde{\mathbf{p}}', \tilde{\mathbf{x}}', \tilde{\mathbf{y}}')$ be the destination produced by the explicit integrator with binding parameter $\omega$ given an initial position $(\tilde{\mathbf{q}}, \tilde{\mathbf{p}},\tilde{\mathbf{q}},\tilde{\mathbf{p}})$. Then $\Vert \tilde{\mathbf{q}}'-\tilde{\mathbf{x}}'\Vert = \mathcal{O}(1/\sqrt{\omega})$ and  $\Vert \tilde{\mathbf{p}}'-\tilde{\mathbf{y}}'\Vert = \mathcal{O}(1/\sqrt{\omega})$.
\end{theorem}
Moreover we note that the numerical error of the integrator satisfies the following bound from \citep{PhysRevE.94.043303} 
\begin{theorem}\label{thm:solution-error}
The error of the integrator is $\mathcal{O}(\min\set{\epsilon^{-2} \cdot \omega^{-1}, \sqrt{\omega}}\cdot \epsilon^{2} \cdot \omega)$ until time $\min\set{\epsilon^{-2} \cdot \omega^{-1}, \sqrt{\omega}}$.
\end{theorem}

\section{Remarks on the Proposal Operator}\label{app:remarks-on-proposal-operator}

We note that in order to have a correct MCMC procedure it is necessary to be able to give a reversible transition operator for the dynamics. We will now develop such a procedure in an asymptotic sense. 

Suppose we have a Hamiltonian $H(\mathbf{q},\mathbf{p})$ in non-canonical coordinates. Our objective is to sample from the distribution $p(\mathbf{q},\mathbf{p})\propto \exp(-H(\mathbf{q},\mathbf{p}))$. To achieve this we will instead derive a Markov chain targeting $p(\mathbf{q},\mathbf{p}, \epsilon) \propto \exp(-H(\mathbf{q},\mathbf{p})) \cdot \frac{1}{2} \cdot \mathbf{1}\set{\epsilon\in\set{-\epsilon^*, +\epsilon^*}}$, where $\epsilon$ is the integration step-size. We will generate samples from $p(\mathbf{q},\mathbf{p})$ by projecting samples from $p(\mathbf{q},\mathbf{p}, \epsilon)$ to their $(\mathbf{q},\mathbf{p})$ marginals.

Let $(\mathbf{q},\mathbf{p})$ be our position in phase-space. Define the doubling and halving operators, respectively, by $\mathbb{D} : (\mathbf{q},\mathbf{p}, \epsilon) \mapsto (\mathbf{q},\mathbf{p}, \mathbf{q},\mathbf{p}, \epsilon)$ and $\mathbb{H} : (\mathbf{q},\mathbf{p}, \mathbf{x},\mathbf{y}, \epsilon) \mapsto (\mathbf{q},\mathbf{p}, \epsilon)$. By \cref{thm:darboux-theorem}, we may find a basis $\mathbf{B}$ and change-of-basis matrix $\mathbf{F}=\mathbf{B}^{-1}$ in which the non-canonical symplectic structure assumes the canonical form. Our development now requires us to further augment the expanded phase-space via the introduction of the integration step-size  $\epsilon$. Applying the change-of-basis operation we set
\begin{align}
    \begin{bmatrix} \tilde{\mathbf{q}}\\\tilde{\mathbf{p}}\\\tilde{\mathbf{x}}\\\tilde{\mathbf{y}}\\\epsilon\end{bmatrix} \defeq \begin{bmatrix} \mathbf{F} &  \\  & \mathbf{F} \\ && 1 \end{bmatrix} \begin{bmatrix} \mathbf{q}\\\mathbf{p}\\\mathbf{x}\\\mathbf{y}\\\epsilon\end{bmatrix}
\end{align}
giving our position in expanded phase-space with respect to canonical coordinates.

We now introduce the integration operator in canonical coordinates $\mathbb{I}_\omega(\tilde{\mathbf{q}},\tilde{\mathbf{p}}, \tilde{\mathbf{x}}, \tilde{\mathbf{y}}, \epsilon) = (\tilde{\mathbf{q}}',\tilde{\mathbf{p}}', \tilde{\mathbf{x}}', \tilde{\mathbf{y}}', \epsilon)$ where $(\tilde{\mathbf{q}}', \tilde{\mathbf{p}}', \tilde{\mathbf{x}}', \tilde{\mathbf{y}}')$ is the output of the explicit integrator introduced in \cref{app:explicit-integration-scheme} applied to the (possibly non-separable) Hamiltonian $\tilde{H}(\tilde{\mathbf{q}},\tilde{\mathbf{p}}) \defeq H(\mathbf{B}(\tilde{\mathbf{q}},\tilde{\mathbf{p}})^\top)$. By passing to the limit $\omega\to+\infty$, we establish by \cref{thm:solution-binding} that $\tilde{\mathbf{q}}'=\tilde{\mathbf{x}}'$ and $\tilde{\mathbf{p}}'=\tilde{\mathbf{y}}'$. (We note that in order to preserve accurate integration of the Hamiltonian, one will require a corresponding decrease in the step-size $\epsilon^* = 1 / \omega$ though this is not strictly necessary for detailed balance to hold; see \cref{thm:solution-error}.)
Further introduce the step-size flip operator $\mathbb{F} : (\tilde{\mathbf{q}},\tilde{\mathbf{p}}, \tilde{\mathbf{x}}, \tilde{\mathbf{y}}, \epsilon) \mapsto (\tilde{\mathbf{q}},\tilde{\mathbf{p}}, \tilde{\mathbf{x}}, \tilde{\mathbf{y}}, -\epsilon)$, which conserves energy and volume and leaves the marginal distribution of $(\mathbf{q},\mathbf{p},\mathbf{x},\mathbf{y})$ invariant. Defining $\mathbb{Q} \defeq \mathbb{F} \circ \mathbb{I}_\omega$, it is evident from symmetry of the explicit integrator that $\mathbb{Q} \circ \mathbb{Q} = \text{Id}$. We now give the full transition operator which starts and ends in non-canonical coordinates (with step-size augmentation):
\begin{align}
    \mathbb{T} \defeq \mathbb{H}\circ \begin{bmatrix} \mathbf{B} &  \\  & \mathbf{B} \\ && 1 \end{bmatrix} \circ \mathbb{Q} \circ \begin{bmatrix} \mathbf{F} &  \\  & \mathbf{F} \\ && 1 \end{bmatrix} \circ \mathbb{D}
\end{align}
which acts as
\begin{align}
    \mathbb{T} : (\mathbf{q},\mathbf{p},\epsilon) \mapsto (\mathbf{q}',\mathbf{p}',-\epsilon)
\end{align}
where
\begin{align}
    \begin{bmatrix} \mathbf{q}'\\\mathbf{p}'\\\mathbf{x}'\\\mathbf{y}'\end{bmatrix} \defeq \begin{bmatrix} \mathbf{B} &  \\  & \mathbf{B} \end{bmatrix} \begin{bmatrix} \tilde{\mathbf{q}}'\\\tilde{\mathbf{p}}'\\\tilde{\mathbf{x}}'\\\tilde{\mathbf{p}}'\end{bmatrix}.
\end{align}
It is easily checked that $\mathbb{T}$ is symmetric.
Note that $\mathbb{T}$ also preserves volume which follows from the fact that $\mathbb{Q}$ is symplectic and the stretching of the expanded phase-space introduced by the change-of-basis to canonical coordinates is undone by the change-of-basis back to non-canonical coordinates. The volumetric expansion and retraction of phase-space given by the operators $\mathbb{D}$ and $\mathbb{H}$ also cancel each other.

Let $R$ be a region of phase-space. Let $R$ be sufficiently small that the value of the acceptance Hamiltonian $G(\mathbf{q},\mathbf{p}) \defeq H(\mathbf{q},\mathbf{p})$ is constant on $R$ with value $G(R)$ and suppose the volume of $R$ is $\Delta$. Let $R'$ be the image of $R$ under $\mathbb{H}$ with step-size $+\epsilon^*$, which has constant Hamiltonian $H(R')$. We obtain the probability of transitioning from $R$ to $R'$:
\begin{align}
    &\int_{R'} \int_{R} \frac{\exp(-H(\alpha))}{Z} \cdot \text{Pr}\left[\alpha' = \mathbb{T}(\alpha, \epsilon^*), \epsilon = +\epsilon^*\right] \d{\alpha'}\d{\alpha} \\ 
    &\qquad = \frac{e^{-H(R)}}{Z} \cdot\Delta\cdot\frac{1}{2}\cdot \min\set{1, e^{H(R') - H(R)}} \\
    &\qquad = \frac{e^{-H(R')}}{Z} \cdot\Delta\cdot\frac{1}{2}\cdot \min\set{1, e^{H(R) - H(R')}} \\
    &\qquad = \int_{R} \int_{R'} \frac{\exp(-H(\alpha'))}{Z}\cdot \text{Pr}\left[\alpha = \mathbb{T}(\alpha', -\epsilon^*), \epsilon = -\epsilon^*\right] \d{\alpha} \d{\alpha'}.
\end{align}
where we have used the shorthand notation $\alpha=(\mathbf{q},\mathbf{p},\mathbf{x},\mathbf{y})$ and $\alpha'=(\mathbf{q}',\mathbf{p}',\mathbf{x}',\mathbf{y}')$. This establishes detailed balance. Samples generated by this procedure, which are subsequently projected to the $(\mathbf{q},\mathbf{p})$ variables only, targets the distribution $p(\mathbf{q},\mathbf{p})\propto \exp(-H(\mathbf{q},\mathbf{p}))$.

Crucially, because $(\tilde{\mathbf{q}}', \tilde{\mathbf{p}}')$ is an approximate destination for Hamilton's equations of motion with Hamiltonian $\tilde{H}(\tilde{\mathbf{q}}, \tilde{\mathbf{p}})$ and initial condition $(\tilde{\mathbf{q}}, \tilde{\mathbf{p}})$, we expect that $\tilde{H}$ will be approximately conserved. Moreover, by \cref{thm:non-canonical-dynamics-darboux} we have that $(\mathbf{q}',\mathbf{p}')$ is an approximate destination for the non-canonical Hamiltonian dynamics with Hamiltonian $H(\mathbf{q},\mathbf{p})$ and initial condition $(\mathbf{q},\mathbf{p})$. It is for this reason that we expect the acceptance probability of the Markov chain to be high.

The question of detailed balance in the case of finite $\omega$ remains open. In our experiments, reversibility of explicit integator is not exact, instead only approximating true reversibility. Hence the use of this explicit integrator can be used as an approximation Hamiltonian Monte Carlo with a reversible integrator. For very small step-sizes, the explicit integrator will exhibit near-reversibility in the sense that, at the end of the trajectory, $\tilde{\mathbf{q}}' \approx\tilde{\mathbf{x}}'$ and $\tilde{\mathbf{p}}' \approx \tilde{\mathbf{y}}'$.

Of course it would be desirable to prove that the Markov chain we have prescribed is actually ergodic, which would require proofs of irreducibility and aperiodicity. Unfortunately, such proofs are not trivial even in the case of canonical Hamiltonian dynamics; refer to \citep{livingstone2016geometric}. We reserve investigation of these properties for future work and henceforth assume that the initial sample is drawn from $p(\mathbf{q},\mathbf{p})$, for which the reversible Markov chain leaves the distribution invariant. If one is concerned about ergodicity, one may instead enter the target distribution using a provably ergodic algorithm (such a Metropolis-adjusted Langevin) and then transition to non-canonical HMC.

\begin{figure*}
    \centering
    \begin{subfigure}[t]{0.2\textwidth}
        \centering
        \includegraphics[width=\textwidth]{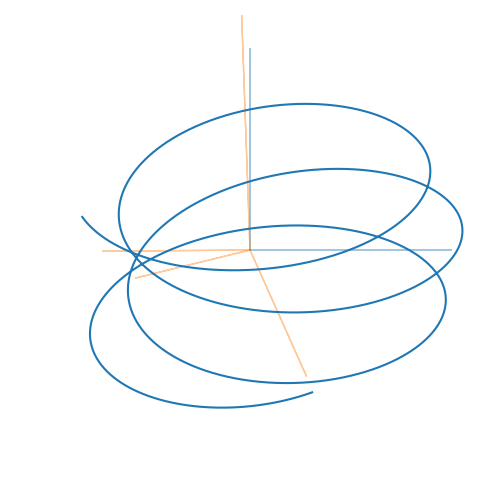}
        \caption{Non-canonical Position}
    \end{subfigure}
    ~
    \begin{subfigure}[t]{0.2\textwidth}
        \centering
        \includegraphics[width=\textwidth]{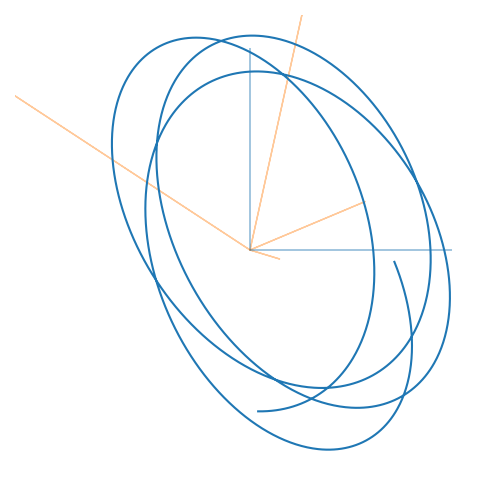}
        \caption{Non-canonical Momentum}
    \end{subfigure}
    ~
    \begin{subfigure}[t]{0.1\textwidth}
    \vspace{-2cm}
    \begin{align*}
        \overset{\tilde{z} = \mathbf{F}z}{\implies} \\
        \underset{z = \mathbf{B}\tilde{z}}{\impliedby}
    \end{align*}
    \end{subfigure}
    ~
    \begin{subfigure}[t]{0.2\textwidth}
        \centering
        \includegraphics[width=\textwidth]{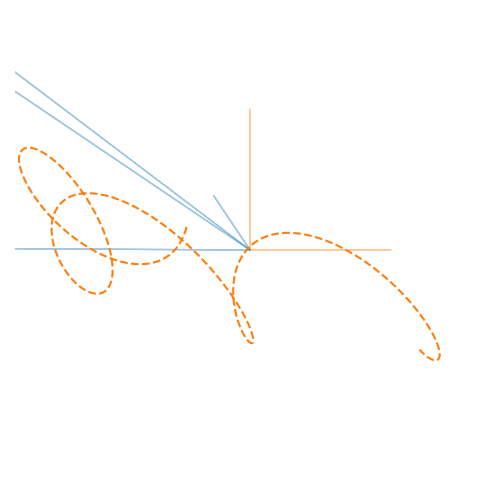}
        \caption{Canonical Position}
    \end{subfigure}
    ~
    \begin{subfigure}[t]{0.2\textwidth}
        \centering
        \includegraphics[width=\textwidth]{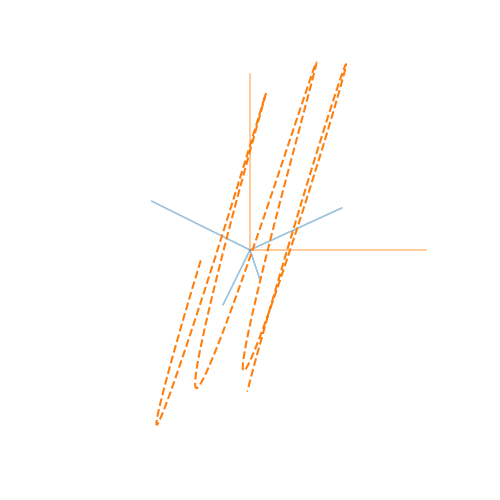}
        \caption{Canonical Momentum}
    \end{subfigure}
    \caption{Visualization of the transformation from non-canonical coordinates to canonical ones via a symplectic basis. In either coordinate system, the motion may be transformed via a linear change-of-basis into the motion of the other coordinate system. The advantage of the symplectic basis is that it is amenable to an explicit integration procedure. Non-canonical motion and its basis are shown in blue while canonical motion and its basis are shown in orange. Basis vectors of the canonical (non-canonical, resp.) coordinate system are shown relative to the non-canonical (canonical) coordinates by light orange (blue) vectors.}
    \label{fig:change-of-basis}
\end{figure*}

\section{Relevance of Binding Strength Parameter}\label{app:binding-strength-parameter}

Given a Hamiltonian and a constant symplectic structure, either the midpoint procedure or the explicit algorithm should be able to integrate the corresponding dynamics (in a symplectic basis in the case of the explicit integrator). However, the explicit integrator introduces an augmented phase-space with additional position and momentum variables; a symplectic integrator is created by tying two solutions, one in $(\mathbf{q}, \mathbf{p})$-space and another in $(\mathbf{x}, \mathbf{y})$-space, together via a binding term whose importance is modulated by $\omega >0$. If the two integrators are correct, we expect the samples generated by an HMC algorithm using either integrator to be close if not identical. By explicitly controlling pseudo-random number generation, we able to examine this property as a function of integration step-size and binding strength.

\begin{figure}[ht!]
    \centering
    \includegraphics[width=0.5\textwidth]{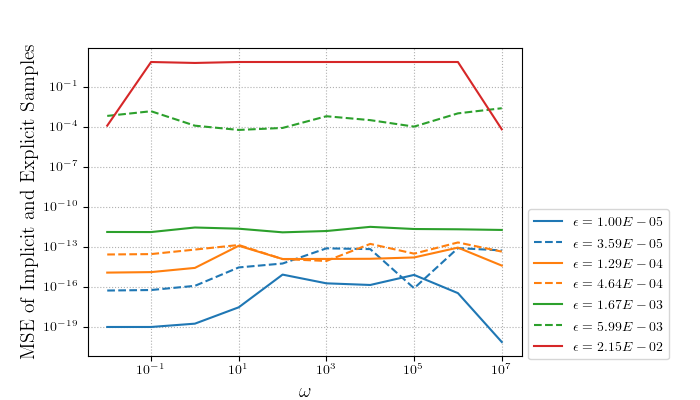}
    \caption{We study the relationship between the explicit integrator and the implicit midpoint integrator of the same system. The explicit procedure depends on the binding strength $\omega>0$; we examine the agreement of the explicit and implicit algorithm across several orders of magnitude in $\omega$. Because the integrated system is the same, including sampled non-canonical momenta, we expect samples generated from HMC using either integrator to be nearly identical. We draw one-thousand samples from a bivariate Gaussian mixture using varying integration step sizes $\epsilon$ as a function of binding strength. We fix the number of integration steps per-sample to also equal one-thousand.}
    \label{fig:omega-err}
\end{figure}

To study this relationship we consider drawing samples from a bivariate Gaussian mixture with unit diagonal variance components and centers at $(2.5, -2.5)$ and $(-2.5, 2.5)$. We consider a non-canonical dynamics produced by drawing standard normal entries of a $4\times 4$ matrix $\mathbf{G}$ and then skew-symmetrizing by the operation $\mathbf{G} \to (\mathbf{G} - \mathbf{G}^\top) / 2$. The relationship between the step-size and $\omega$ is shown in \cref{fig:omega-err}. We find that samples generated by either integrator are within agreement over several orders of magnitude in either step-size or binding strength when the number of integration steps is held constant. In our remaining experiments, we set $\omega = 1$.

\section{Reversibility of the Implicit Integrator}\label{app:implicit-reversibility}

{\bf Theorem 4.} Let $\Lambda$ be a Poisson structure with time-reversal Poisson structure $\tilde{\Lambda}$. Suppose $H(\mathbf{q},\mathbf{p}) = U(\mathbf{q}) + \frac{1}{2} \mathbf{p}^\top\mathbf{p}$ is a separable Hamiltonian. The implicit midpoint integrator satisfies detailed balance when combined with a momentum-flip operator and a transition to the time-reversal Poisson structure as described in \cref{thm:time-reversal-poisson}.

\begin{proof}
When using an implicit integrator (see \cref{app:numerical-integration}), it can be shown that reversibility in the presence of a non-canonical Poisson structure is achievable by negating the momentum and subsequently reverting to the time-reversal Poisson structure. We show this fact directly. First observe that the implicit update satisfies the equation for a separable Hamiltonian $H(\mathbf{q},\mathbf{p}) = U(\mathbf{q}) + \frac{1}{2} \mathbf{p}^\top\mathbf{p}$,
\begin{align}
    \begin{bmatrix}\mathbf{q}_1\\\mathbf{p}_1\end{bmatrix} &= \begin{bmatrix}\mathbf{q}_0\\\mathbf{p}_0\end{bmatrix} + \epsilon \cdot X_H\paren{\frac{\mathbf{q}_0 + \mathbf{q}_1}{2},\frac{\mathbf{p}_0 +\mathbf{p}_1}{2}} \\
    &= \begin{bmatrix}\mathbf{q}_0\\\mathbf{p}_0\end{bmatrix} + \epsilon\cdot \begin{bmatrix} \mathbf{E} & \mathbf{F} \\ -\mathbf{F}^\top & \mathbf{G}\end{bmatrix} \begin{bmatrix} \nabla_\mathbf{p} H\paren{\frac{\mathbf{q}_0 + \mathbf{q}_1}{2},\frac{\mathbf{p}_0 +\mathbf{p}_1}{2}} \\ \nabla_\mathbf{q} H\paren{\frac{\mathbf{q}_0 + \mathbf{q}_1}{2},\frac{\mathbf{p}_0 +\mathbf{p}_1}{2}} \end{bmatrix} \\
    &= \begin{bmatrix}\mathbf{q}_0\\\mathbf{p}_0\end{bmatrix} + \epsilon\cdot\begin{bmatrix} \mathbf{E}\nabla_q U\paren{\frac{\mathbf{q}_0 + \mathbf{q}_1}{2}} + \mathbf{F}\paren{\frac{\mathbf{p}_0 +\mathbf{p}_1}{2}} \\
    -\mathbf{F}^\top \nabla_q U\paren{\frac{\mathbf{q}_0 + \mathbf{q}_1}{2}} +\mathbf{G}\paren{\frac{\mathbf{p}_0 +\mathbf{p}_1}{2}} \end{bmatrix}
\end{align}
The claim is that upon reversing the momentum and integrating again with the time-reversal Poisson structure returns to the original position $(\mathbf{q}_0,\mathbf{p}_0)$. To show this, first observe that negating the momentum becomes,
\begin{align}
    \begin{bmatrix}\mathbf{q}_1\\-\mathbf{p}_1\end{bmatrix} = \begin{bmatrix}\mathbf{q}_0\\-\mathbf{p}_0\end{bmatrix} + \epsilon\cdot\begin{bmatrix} \mathbf{E}\nabla_q U\paren{\frac{\mathbf{q}_0 + \mathbf{q}_1}{2}} + \mathbf{F}\paren{\frac{\mathbf{p}_0 +\mathbf{p}_1}{2}} \\
    \mathbf{F}^\top \nabla_q U\paren{\frac{\mathbf{q}_0 + \mathbf{q}_1}{2}} - \mathbf{G}\paren{\frac{\mathbf{p}_0 +\mathbf{p}_1}{2}} \end{bmatrix}
\end{align}
Recall that the time reversal Poisson structure assumes the form,
\begin{align}
    \begin{bmatrix} -\mathbf{E} & \mathbf{F} \\ -\mathbf{F}^\top & -\mathbf{G}\end{bmatrix}.
\end{align}
We now verify that $(\mathbf{q}_0,-\mathbf{p}_0)$ is stationary for the implicit update with the reversed momentum and time-reversal Poisson structure.
\begin{align}
    \begin{bmatrix}\mathbf{q}_2\\ \mathbf{p}_2\end{bmatrix} &=
    \begin{bmatrix}\mathbf{q}_1\\ -\mathbf{p}_1\end{bmatrix} +
    \epsilon\cdot\begin{bmatrix} -\mathbf{E}\nabla_q U\paren{\frac{\mathbf{q}_1 + \mathbf{q}_2}{2}} + \mathbf{F}\paren{\frac{-\mathbf{p}_1 +\mathbf{p}_2}{2}} \\
    -\mathbf{F}^\top \nabla_q U\paren{\frac{\mathbf{q}_1 + \mathbf{q}_2}{2}}  -\mathbf{G}\paren{\frac{-\mathbf{p}_1 +\mathbf{p}_2}{2}} \end{bmatrix}  \\
    =& \begin{bmatrix}\mathbf{q}_0\\-\mathbf{p}_0\end{bmatrix} + \epsilon\cdot\begin{bmatrix} \mathbf{E}\nabla_q U\paren{\frac{\mathbf{q}_0 + \mathbf{q}_1}{2}} + \mathbf{F}\paren{\frac{\mathbf{p}_0 +\mathbf{p}_1}{2}} \\
    \mathbf{F}^\top \nabla_q U\paren{\frac{\mathbf{q}_0 + \mathbf{q}_1}{2}} - \mathbf{G}\paren{\frac{\mathbf{p}_0 +\mathbf{p}_1}{2}} \end{bmatrix} +
    \epsilon\cdot\begin{bmatrix} -\mathbf{E}\nabla_q U\paren{\frac{\mathbf{q}_1 + \mathbf{q}_2}{2}} + \mathbf{F}\paren{\frac{-\mathbf{p}_1 +\mathbf{p}_2}{2}} \\
    -\mathbf{F}^\top \nabla_q U\paren{\frac{\mathbf{q}_1 + \mathbf{q}_2}{2}}  -\mathbf{G}\paren{\frac{-\mathbf{p}_1 +\mathbf{p}_2}{2}} \end{bmatrix}
\end{align}
By inspection, we observe that the choice $\mathbf{q}_2 = \mathbf{q}_0$ and $\mathbf{p}_2 = -\mathbf{p}_0$ solves the implicit relation. We have used the observation that $(-\mathbf{p}_1 - \mathbf{p}_0)/2 = -(\mathbf{p}_1+\mathbf{p}_0)/2$. This establishes symmetry of the operator used for implicit integration. Standard arguments apply to show that detailed balance holds for non-canonical HMC with the implicit midpoint algorithm used as a transition mechanism. One only needs to equip the integrator trajectory with momentum flip and Poisson structure time-reversal operators; or, equivalently, to equip the transition operator with a random choice to integrate with a positive or negative step-size of equal magnitude: $\text{Pr}\left[\epsilon = +\epsilon^*\right] = \text{Pr}\left[\epsilon=-\epsilon^*\right] = 1/2$ for some $\epsilon^* \neq 0$.
\end{proof}

\end{document}